\PassOptionsToPackage{hyphens}{url}
\documentclass[sigconf]{acmart}
\AtBeginDocument{}
\copyrightyear{2026}
\acmYear{2026}
\setcopyright{cc}
\setcctype{by}
\acmConference[KDD '26]{Proceedings of the 32nd ACM SIGKDD Conference on Knowledge Discovery and Data Mining V.2}{August 09--13, 2026}{Jeju Island, Republic of Korea}
\acmBooktitle{Proceedings of the 32nd ACM SIGKDD Conference on Knowledge Discovery and Data Mining V.2 (KDD '26), August 09--13, 2026, Jeju Island, Republic of Korea}
\acmDOI{10.1145/3770855.3818835}
\acmISBN{979-8-4007-2259-2/2026/08}

\usepackage{graphicx}
\usepackage{natbib}
\usepackage{caption}
\usepackage{algorithm}
\usepackage{algpseudocode}  
\usepackage{multirow}

\usepackage{amssymb}
\usepackage{subcaption}
\usepackage{booktabs}
\usepackage{amsmath}
\usepackage{amsthm}
\usepackage{newfloat}
\usepackage{listings}

\usepackage{enumitem}
\DeclareCaptionStyle{ruled}{labelfont=normalfont,labelsep=colon,strut=off}
\floatstyle{ruled}
\newfloat{listing}{tb}{lst}{}
\floatname{listing}{Listing}

\newtheorem{proposition}{Proposition}

\newtheorem{assumption}{Assumption}

\title{Lightweight Wrappers for Adapting Time Series Foundation Models to Regional Drought Forecasting}

\author{Wentao Gao} \authornote{Corresponding author. Email: wentao.gao@adelaide.edu.au}
\affiliation{%
  \institution{Adelaide University}
  \city{Adelaide}
  \country{Australia}
}

\author{Jiuyong Li}
\affiliation{%
  \institution{Adelaide University}
  \city{Adelaide}
  \country{Australia}
}

\author{Lin Liu}
\affiliation{%
  \institution{Adelaide University}
  \city{Adelaide}
  \country{Australia}
}

\author{Thuc Duy Le}
\affiliation{%
  \institution{Adelaide University}
  \city{Adelaide}
  \country{Australia}
}

\author{Jixue Liu}
\affiliation{%
  \institution{Adelaide University}
  \city{Adelaide}
  \country{Australia}
}

\author{Yanchang Zhao}
\affiliation{%
  \institution{CSIRO Technology}
  \city{Adelaide}
  \country{Australia}
}

\author{Yun Chen}
\affiliation{%
  \institution{CSIRO Environment}
  \city{Adelaide}
  \country{Australia}
}

\renewcommand{\shortauthors}{Wentao Gao et al.}
\usepackage{bibentry}

\begin{document}

\begin{abstract}
Large \emph{Time Series Foundation Models} (TSFMs) demonstrate strong zero-shot forecasting capabilities across diverse domains. However, their application to regional climate forecasting faces practical challenges: model weights are often proprietary, local training records are limited, and computational budgets are constrained, making traditional fine-tuning approaches infeasible. To address these constraints, we introduce a lightweight, black-box adaptation framework (requiring no access to backbone parameters and no backbone fine-tuning) that enhances frozen TSFMs at inference time through two plug-and-play wrappers: \textbf{SMR\textsuperscript{2}} (Stationarity aware multi-resolution Residual), which decomposes the input into multi-resolution temporal views, learns stride specific residual corrections that capture regional dynamics, then adaptively ensembles them into a single forecast, and \textbf{MBB} (Moving Block Bootstrap), which preserves temporal dependencies through block resampling and ensembles over temporally coherent residual perturbations to stabilize the point forecast. Both wrappers instantiate the same bagging style principle: they build diverse views of the input or its residuals, forecast each with the same frozen backbone, and aggregate, so all adaptation comes from inference time ensembling rather than any weight update. Evaluated on one month ahead Standardized Precipitation Evapotranspiration Index (SPEI) prediction across multiple sites in South Australia, our framework consistently improves forecasting performance across several backbone models, demonstrating up to 26\% mean squared error (MSE) reduction over the corresponding frozen backbone while enabling practical deployment in resource constrained regional forecasting systems.
\end{abstract}

\ccsdesc[500]{Computing methodologies~Machine learning}
\ccsdesc[300]{Applied computing~Earth and atmospheric sciences}
\ccsdesc[300]{Mathematics of computing~Time series analysis}

\keywords{time series foundation models, inference time adaptation, drought forecasting, SPEI, multi-resolution analysis, block bootstrap, black-box model adaptation}

\maketitle

\section{Introduction}

Drought is among the most disruptive climate hazards, especially in semi arid regions such as South Australia, where it undermines agriculture, water security, and ecosystem resilience \cite{Mishra2010}. The Standardized Precipitation Evapotranspiration Index (SPEI) integrates precipitation with potential evapotranspiration and is widely used to characterise drought severity across time scales \cite{Vicente2010,Begueria2014}.

Existing approaches face challenges in predicting SPEI one to three months in advance, especially in local regions where data is insufficient to train complex models. Traditional statistical models including ARIMA, SARIMA, and state space formulations \cite{Box2015TimeSeriesAnalysis} assume quasi stationarity and therefore struggle to capture the coexistence of short lived pulses and long term trends. Traditional machine learning algorithms such as support vector machines or random forests can model nonlinear relationships yet depend on hand crafted features and suffer in high dimensional, multivariate settings \cite{belgiu2016random}. Deep architectures LSTMs \cite{Hochreiter1997}, TCNs \cite{Bai2018}, and the growing family of Transformer variants \cite{vaswani2017attention,wu2023timesnet_iclr} can learn temporal dependencies end to end, but they require large datasets and are prone to overfitting in data sparse regions.

\begin{figure}[h]
    \centering
    \includegraphics[width=1\linewidth]{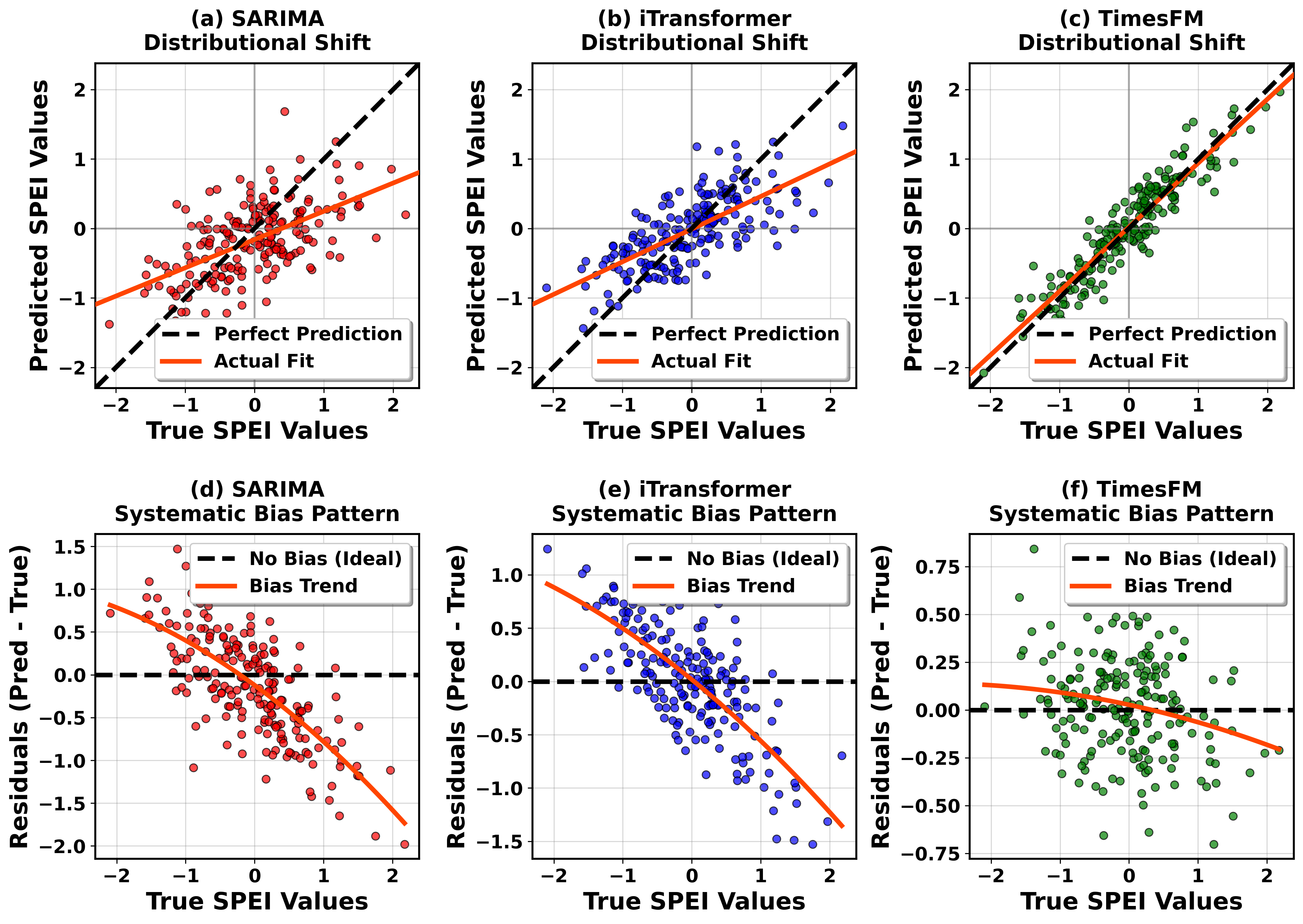}
    \caption{Regional forecasting limitations across three model families. (a to c) Predicted vs.\ true SPEI for SARIMA, iTransformer, and TimesFM. The three behave differently because they encode different inductive biases: SARIMA assumes quasi stationarity, iTransformer is task-trained on the available regional data in our experiment, and TimesFM is applied zero-shot from large scale pre training. (d to f) Corresponding residuals. All three exhibit \emph{systematic} (non random) regional bias rather than white noise; the bias is pronounced for SARIMA and iTransformer (d, e) and small but still structured for TimesFM (f). These systematic regional biases motivate our inference-time correction.}
    \label{fig:enter_label}
\end{figure}

Large Time Series Foundation Models (TSFMs) have emerged as a promising alternative.  
Pre trained on millions of heterogeneous sequences, these models demonstrate strong zero-shot skill across finance, traffic, and meteorology \cite{Bommasani2021}; recent examples include TimesFM \cite{Sen2024} and Lag-Llama \cite{rasul2023lagllama}.  
While physics aware encoder decoder designs such as GraphCast \cite{Lam2022}, FourCastNet \cite{Pathak2022}, and Pangu-Weather \cite{Bi2023} have advanced global numerical weather prediction, they require complex post processing pipelines to generate drought indices, making general purpose TSFMs more attractive for direct climate index forecasting.  

Yet applying TSFMs to regional drought forecasting faces significant practical constraints that traditional adaptation approaches cannot address. Foundation model weights are often proprietary and inaccessible, preventing direct fine-tuning or architectural modifications. Regional climate records are typically short spanning decades rather than the centuries needed for robust model retraining while operational forecasting centers operate under strict computational budgets that preclude resource intensive adaptation methods. These constraints necessitate lightweight, inference time solutions that can enhance foundation model performance without requiring parameter access or extensive retraining.

Figure 1 illustrates these regional adaptation difficulties across three representative model families: a classical statistical model (SARIMA), a deep learning model (iTransformer), and a foundation model (TimesFM). The three differ because they rely on different inductive biases, yet all leave \emph{systematic} rather than random residuals. The bias is most pronounced for SARIMA and iTransformer and smallest for the pre trained TimesFM, but even TimesFM retains a structured residual pattern. This indicates that regional temporal dynamics contain characteristics at different temporal resolutions that large scale pre training does not fully capture, motivating specialized inference time adaptation.

To address these practical constraints while preserving foundation models' general capabilities, we introduce two lightweight inference time wrappers. The Stationarity aware multi-resolution Residual wrapper (SMR\textsuperscript{2}) decomposes the regional forecasting task into multi-resolution temporal pattern learning, capturing stride specific regional dynamics that complement the global knowledge embedded in foundation models. MBB (Moving Block Bootstrap) constructs residual block perturbations of the target history and averages the corresponding frozen backbone forecasts.

In summary, our main contributions are as follows:
\begin{itemize}
    \item We propose two plug-and-play inference time wrappers, SMR\textsuperscript{2} for multi-resolution regional pattern learning and MBB for variance reducing block bootstrap ensembling, that require no access to the backbone foundation model weights.
    \item We present a theoretical analysis built on a single residual decomposition assumption: a convex risk bound with an explicit bias variance trade off that motivates the adaptive multi-resolution weighting, together with the standard block bootstrap variance identity that explains the ensemble's stabilizing effect.
    \item We conduct comprehensive experiments on SPEI-30 forecasting across three South Australian sites, showing up to 26\% mean squared error (MSE) reduction over the corresponding frozen backbone, consistently outperforming 14 statistical and deep learning baselines across three foundation model backbones.
\end{itemize}

\section{Related Work}

\textbf{Drought Forecasting Approaches.} Traditional statistical methods like ARIMA assume stationarity and struggle with nonlinear climate dynamics \cite{Mishra2010}, while machine learning approaches (SVMs, Random Forests) require handcrafted features and lack temporal modeling capabilities \cite{belgiu2016random}. Deep learning architectures including LSTMs \cite{Hochreiter1997}, TCNs \cite{Bai2018}, and Transformers \cite{vaswani2017attention} enable end to end learning but require task specific retraining and suffer from limited regional data availability.

\textbf{Time Series Foundation Models.} Large scale pre trained models like TimesFM \cite{Sen2024} and Lag-Llama \cite{rasul2023lagllama} demonstrate strong zero-shot capabilities across domains. While domain specific weather models (GraphCast \cite{Lam2022}, FourCastNet \cite{Pathak2022}) excel at global numerical prediction, they generate gridded physical variables requiring extensive post processing for climate indices, introducing uncertainty sources and computational constraints that exceed regional system capabilities. This motivates our focus on general purpose TSFMs that can forecast climate indices directly, with regional adaptation handled outside the backbone.

\textbf{Multi-resolution and Resampling Methods.} multi-resolution analysis decomposes signals across temporal scales \cite{Percival2000Wavelet}, with recent advances in multi scale attention \cite{guo2025multivariate} and residual frameworks like N-BEATS \cite{Oreshkin2020NBeats}. The Moving Block Bootstrap \cite{kunsch1989jackknife} preserves temporal dependencies when resampling serially correlated data \cite{kokic2013blockbootstrap}. These methods are therefore less suited to black-box backbone adaptation.

Our work builds on these ideas by applying multi-resolution residual correction and temporally aware resampling at inference time around a frozen backbone.

\section{Preliminary}
\label{sec:preliminary}

\subsection{Problem Formulation}

Let $\{(\mathbf{c}_t, y_t)\}$ denote a multivariate climate time series, where each $\mathbf{c}_t \in \mathbb{R}^{d}$ contains $d=38$ predictor features (34 meteorological variables and 4 climate indices) and $y_t \in \mathbb{R}$ is the scalar target SPEI value. For a time interval $a{:}b$, $y_{a:b}=(y_a,\ldots,y_b)$ denotes the corresponding target history. We define:
\begin{itemize}
    \item $L$: Length of the context window used for prediction
    \item $L_{\text{train}}$: Reduced context length used during calibration phase ($L_{\text{train}} < L$)
    \item $\mathcal{T}_{\text{train}}, \mathcal{T}_{\text{val}}$: Index sets used for calibration and validation; held out test splits are used only for reporting
\end{itemize}

At each forecast time $t$, we collect the context window $\mathbf{X}_t = (\mathbf{c}_{t-L+1:t}, y_{t-L+1:t})$ and predict the next SPEI value with a frozen Time Series Foundation Model backbone $f_{\theta}^{\text{base}}$, i.e., a fixed mapping $\hat{y}_{t+1} = f_{\theta}^{\text{base}}(\mathbf{X}_t)$ from a context window to a forecast. We write $\mathbf{X}_t^{(s)}$ for the same window after the stride $s$ multi-resolution view construction is applied to every channel of $\mathbf{X}_t$ (the $d$ covariates and the target history), i.e.\ $\mathbf{X}_t^{(s)} = (\mathbf{c}_{t-L+1:t}^{(s)}, y_{t-L+1:t}^{(s)})$, pairing the stride $s$ covariate view $\mathbf{c}_{t-L+1:t}^{(s)}$ with the stride $s$ target history view $y_{t-L+1:t}^{(s)}$, both constructed in the Method section below. We focus on one month ahead prediction.

\begin{figure*}[h]
    \centering
    \includegraphics[width=0.85\linewidth]{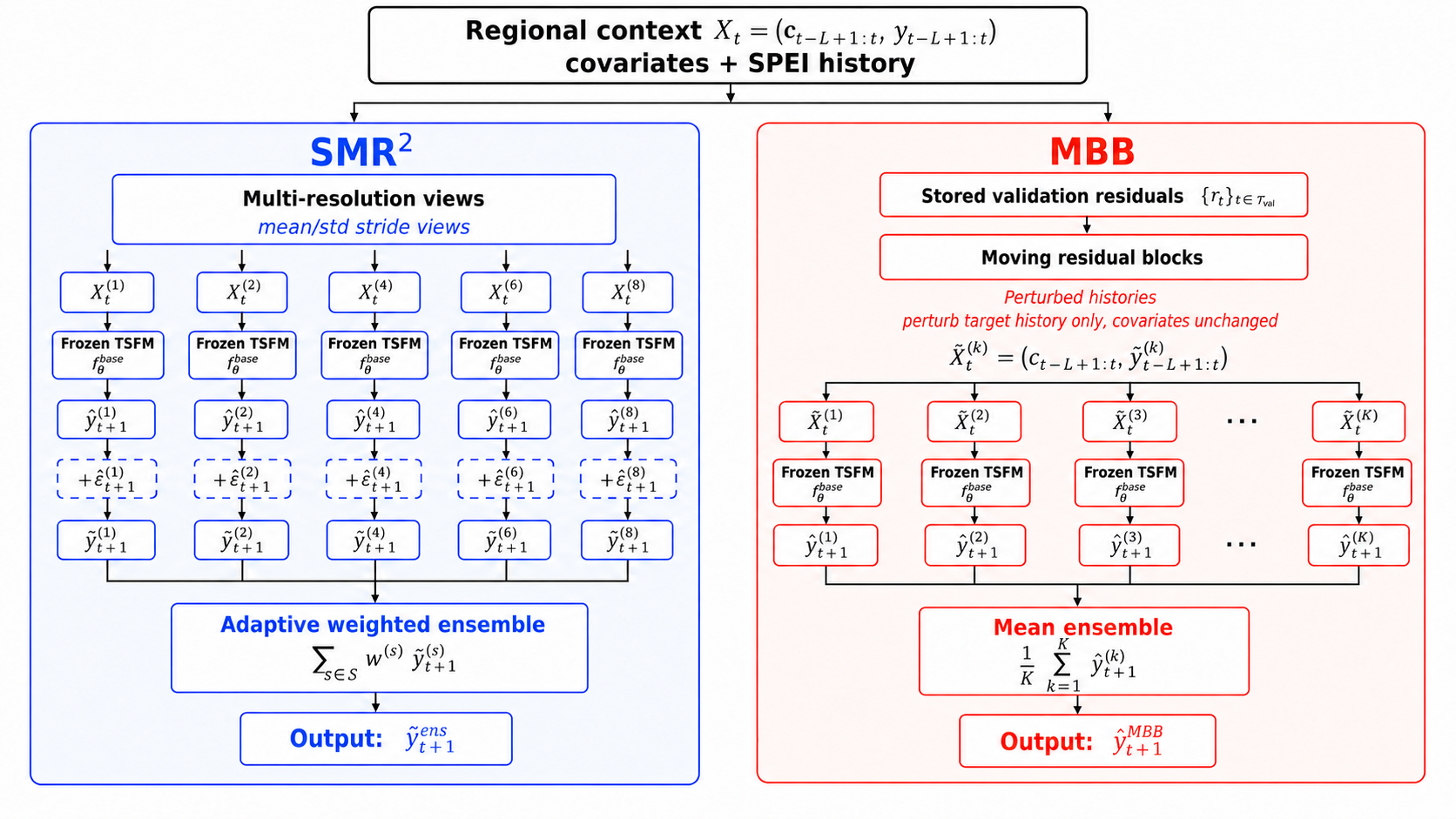}
    \caption{Overview of the proposed inference time wrappers for frozen TSFMs. SMR\textsuperscript{2} forms multi-resolution residual ensemble members from mean/std stride views and adaptively aggregates their corrected forecasts, while MBB forms block bootstrap ensemble members from temporally coherent residual perturbations and averages their forecasts. Both wrappers follow a bagging style principle: diverse inference time members are passed through the same frozen backbone and then aggregated.}
    \label{fig:method_overview}
\end{figure*}

\section{Method}
\label{sec:method}

\subsection{Framework Overview}

To adapt frozen Time Series Foundation Models (TSFMs) for regional forecasting tasks, we propose two lightweight inference time wrappers that operate without modifying pre trained model parameters. Here, \(f_{\theta}^{\text{base}}\) denotes the frozen TSFM backbone (e.g., TimesFM, TimeGPT) that provides the base predictions; we write $f_{\theta}^{\text{base}}$ throughout the Method to emphasize that it is not updated. The SMR\textsuperscript{2} wrapper addresses regional pattern learning, while the MBB wrapper improves forecast robustness through block bootstrap ensembling. Figure~\ref{fig:method_overview} provides an overview of the two wrappers, highlighting their shared bagging style structure: diverse inference time series, a shared frozen backbone, and aggregation.

\subsection{\texorpdfstring{SMR\textsuperscript{2}}{SMR2} Wrapper}

Foundation models trained on global datasets face adaptation difficulties when applied to regional forecasting tasks, owing to practical deployment constraints and the need to capture local temporal dynamics. Figure~\ref{fig:smr2_motivation} demonstrates that regional time series exhibit non stationary shifts over time, different temporal strides reveal distinct distributional characteristics, and foundation model residuals show complementary error patterns across resolutions.

\begin{figure}[t]
    \centering
    \includegraphics[width=1\linewidth]{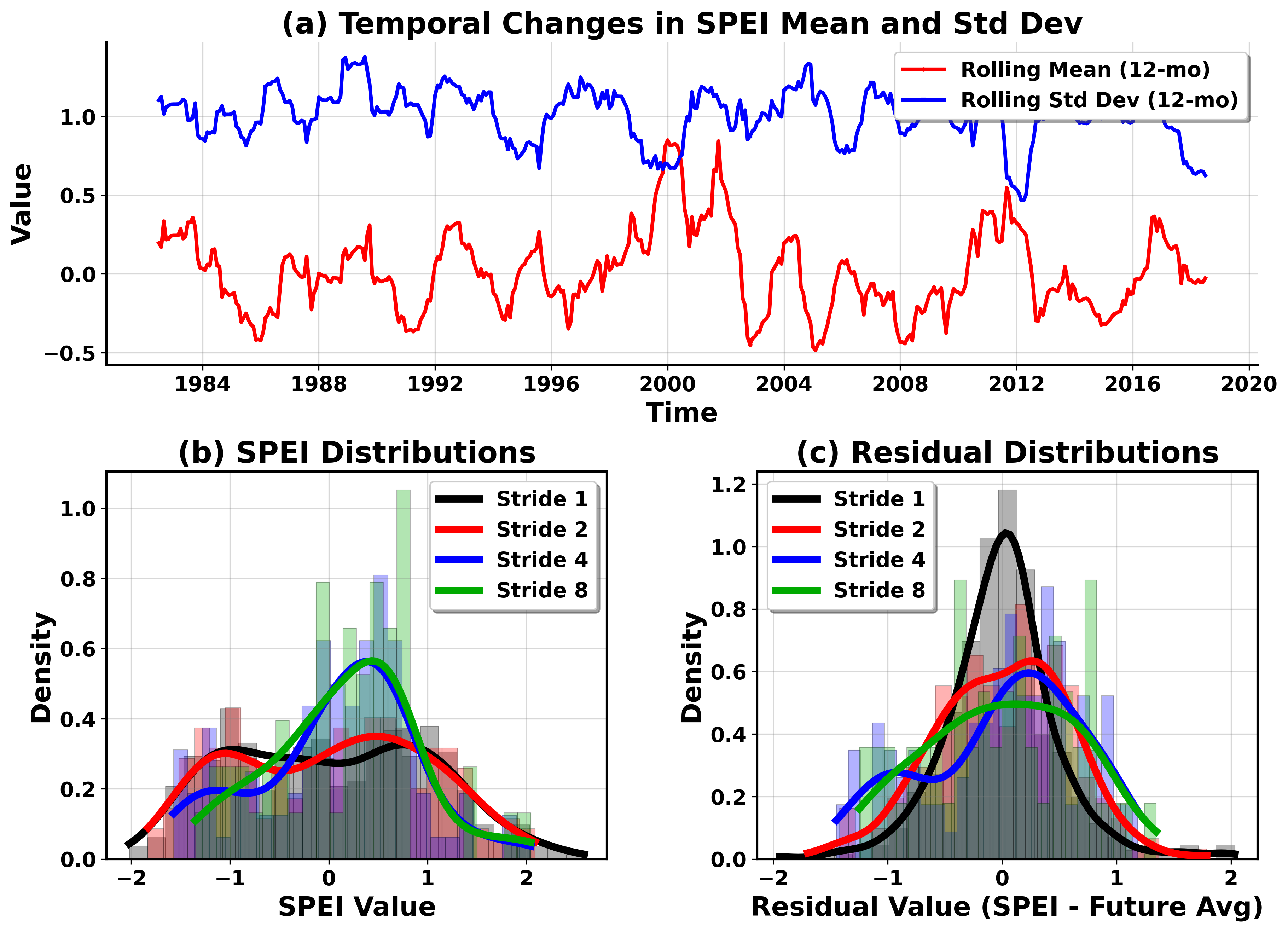}
    \caption{Motivation for multi-resolution regional pattern learning. (a) Rolling 12 month mean (red) and standard deviation (blue) of SPEI over 1982--2018, showing pronounced temporal non stationarity. (b, c) Histograms (bars) with overlaid kernel density estimates (solid curves) of the SPEI values and of the foundation model residuals ($y-\hat{y}$, i.e., observed SPEI minus the base model forecast), respectively. In both panels color encodes the temporal stride used for downsampling: stride~1 (black), 2 (red), 4 (blue), and 8 (green). Different strides yield visibly different distributions, with the residual densities changing shape across resolutions, indicating complementary information across resolutions and motivating multi-resolution adaptation.}
    \label{fig:smr2_motivation}
\end{figure}

Instead of retraining the entire model, we can enhance regional forecasting by learning complementary temporal patterns across multiple resolutions that supplement the foundation model's global knowledge. Different temporal resolutions capture different aspects of the underlying signal: fine grained patterns reveal short term fluctuations while coarse grained patterns expose long term trends. We decompose the regional enhancement problem into multi-resolution pattern learning, where each temporal stride contributes differently to the overall performance, and residual learning, where we learn a residual correction for each stride. We call the wrapper \emph{stationarity aware} in this specific sense: the per segment mean and standard deviation used to build each stride view (Eqs.~\eqref{eq:seg_mean}--\eqref{eq:seg_std}) act as a local normalization that absorbs slow drifts in level and variance \citep{liu2022nonstationary}, rather than because it performs an explicit stationarity test, change point detection, or regime classification.

The SMR\textsuperscript{2} approach addresses both aspects simultaneously. Let $\mathcal{S}$ denote the set of temporal strides and let $w^{(s)}$ be the ensemble weight assigned to stride $s$:
\begin{equation}
\label{eq:smr2_forecast}
\begin{aligned}
\tilde{y}_{t+1}^{\mathrm{ens}} &= \sum_{s \in \mathcal{S}} w^{(s)}\, \tilde{y}_{t+1}^{(s)}, \\
\tilde{y}_{t+1}^{(s)} &= f_{\theta}^{\text{base}}(\mathbf{X}_t^{(s)}) + \hat{\epsilon}_{t+1}^{(s)}.
\end{aligned}
\end{equation}
where $\mathbf{X}_t^{(s)}$ is the stride $s$ multi-resolution input view (of both the covariates and the target history), $\tilde{y}_{t+1}^{(s)}$ is the corrected stride $s$ forecast, $\tilde{y}_{t+1}^{\mathrm{ens}}$ is the final SMR\textsuperscript{2} ensemble forecast, and $\hat{\epsilon}_{t+1}^{(s)}$ is the learned residual correction defined by the two phase procedure below.

\paragraph{multi-resolution Input Construction}
We build the stride $s$ input $\mathbf{X}_t^{(s)}$ by applying multi-resolution downsampling to the full context window $\mathbf{X}_t = (\mathbf{c}_{t-L+1:t}, y_{t-L+1:t})$, that is, to all $d{+}1$ channels at once. Let $\mathbf{x}_\tau = (\mathbf{c}_\tau, y_\tau) \in \mathbb{R}^{d+1}$ denote the input at a single time step $\tau$, stacking the $d$ covariates with the scalar target. For each stride $s$, we partition the first $s\lfloor L/s \rfloor$ positions of the window into $\lfloor L/s \rfloor$ complete non overlapping segments. Any trailing remainder of length $L-s\lfloor L/s \rfloor$ is not used to form an additional partial segment when computing the downsampled statistics. Writing $I_j^{(s)} = \{(j-1)s+1, \ldots, js\}$ for the index set of segment $j$, we compute, for each $j \in \{1, 2, \ldots, \lfloor L/s \rfloor\}$, the per segment mean and standard deviation of the input:
\begin{align}
\bar{\mathbf{x}}_j^{(s)} &= \frac{1}{|I_j^{(s)}|} \sum_{k \in I_j^{(s)}} \mathbf{x}_{t-L+k} \label{eq:seg_mean} \\
\boldsymbol{\sigma}_j^{(s)} &= \sqrt{\frac{1}{|I_j^{(s)}|} \sum_{k \in I_j^{(s)}} \left(\mathbf{x}_{t-L+k} - \bar{\mathbf{x}}_j^{(s)}\right)^2} \label{eq:seg_std}
\end{align}
Both operations act elementwise across the $d{+}1$ channels, so the same stride $s$ summary is applied to every covariate and to the target history alike.

Following downsampling, we perform temporal upsampling to obtain a length $L$ representation. In the implementation, each segment's statistics are repeated $\lfloor L/\lfloor L/s\rfloor\rfloor$ times. If the repeated sequence is shorter than $L$, the remaining positions are filled with the last complete segment statistic; if it is longer than $L$, it is truncated. Thus, for the upsampled feature vector $\mathbf{u}_{t-L+k}^{(s)}$,
\begin{equation}
\begin{aligned}
\mathbf{u}_{t-L+k}^{(s)}
&=
[\bar{\mathbf{x}}_{j}^{(s)},\boldsymbol{\sigma}_{j}^{(s)}],
\\
j=
&
\min\left(
\left\lceil
\frac{k}{\lfloor L/\lfloor L/s\rfloor\rfloor}
\right\rceil,
\left\lfloor \frac{L}{s}\right\rfloor
\right),
\\
&\quad k=1,\ldots,L .
\end{aligned}
\label{eq:upsample}
\end{equation}
Concatenating mean with standard deviation provides richer geometric information than traditional delay embeddings, and stacking the upsampled vectors $\mathbf{u}_{t-L+k}^{(s)}$ over $k$ yields the multi-resolution input view $\mathbf{X}_t^{(s)} \in \mathbb{R}^{L \times 2(d+1)}$ that is passed to the backbone. We write $\mathbf{c}_{t-L+1:t}^{(s)} \in \mathbb{R}^{L \times 2d}$ and $y_{t-L+1:t}^{(s)} \in \mathbb{R}^{L \times 2}$ for its covariate and target history sub blocks, so $\mathbf{X}_t^{(s)} = (\mathbf{c}_{t-L+1:t}^{(s)}, y_{t-L+1:t}^{(s)})$.

\paragraph{Two Phase Prediction Framework}
The two-phase framework separates residual learning from final forecasting.
In the calibration phase, the frozen backbone is applied to each stride view
to obtain a base forecast, and the difference between the observed target and
this base forecast is stored as a stride-specific residual. In the inference
phase, the same frozen backbone first produces the stride-specific base forecast
\(\hat{y}^{(s)}_{t+1}\); a lightweight residual corrector then predicts
\(\hat{\epsilon}^{(s)}_{t+1}\) from recent residuals and current stride-view
summaries; the corrected stride forecast is finally obtained as
\(\tilde{y}^{(s)}_{t+1}=\hat{y}^{(s)}_{t+1}+\hat{\epsilon}^{(s)}_{t+1}\).
We use a reduced context length \(L_{\mathrm{train}}\) during calibration to
compute initial residuals, while using the full context length \(L\) for final
ensemble predictions.

During the calibration phase, we use the reduced context length $L_{\text{train}}$ to compute calibration residuals for all $t \in \mathcal{T}_{\text{train}}$. To avoid confusing this reduced window with the full inference window, let $\mathbf{X}_{t-1,L_{\text{train}}}=(\mathbf{c}_{t-L_{\text{train}}:t-1}, y_{t-L_{\text{train}}:t-1})$ denote the length-$L_{\text{train}}$ context ending at $t-1$, and let $\mathbf{X}_{t-1,L_{\text{train}}}^{(s)}$ be its stride $s$ view:
\begin{equation}
\epsilon_t^{(s)} = y_t - f_{\theta}^{\text{base}}(\mathbf{X}_{t-1,L_{\text{train}}}^{(s)})
\end{equation}
Here $\epsilon_t^{(s)}$ is the realized base model residual at stride $s$. Collecting these residuals gives the calibration residual set $\mathcal{R}^{(s)}=\{\epsilon_t^{(s)}\}_{t\in\mathcal{T}_{\text{train}}}$. Calibration shares the forecast origin alignment of inference and uses no information from time $t$ or later (no look ahead).

During the inference phase, for each stride $s \in \mathcal{S}$, we run both models in parallel using the full context window $L$:
\begin{align}
\hat{y}_{t+1}^{(s)} &= f_{\theta}^{\text{base}}(\mathbf{X}_t^{(s)}), \\
\hat{\epsilon}_{t+1}^{(s)} &= f^{\text{res}}\big(\boldsymbol{\phi}_t^{(s)}\big),
   \qquad \boldsymbol{\phi}_t^{(s)} = \big[\,\epsilon_{t-q+1:t}^{(s)},\; \psi(\mathbf{X}_{t}^{(s)})\,\big], \\
\tilde{y}_{t+1}^{(s)} &= \hat{y}_{t+1}^{(s)} + \hat{\epsilon}_{t+1}^{(s)},
\end{align}
where the corrector input $\boldsymbol{\phi}_t^{(s)}$ concatenates the $q$ most recent stride $s$ residuals $\epsilon_{t-q+1:t}^{(s)}$ with summary statistics $\psi(\cdot)$ of the stride $s$ input view $\mathbf{X}_t^{(s)}$ (the per channel means of its mean and standard deviation components, over both the covariates and the target history). The corrector $f^{\text{res}}$ is a ridge regression with \(\ell_2\) penalty $\lambda_{\text{ridge}}=0.1$, fit \emph{once} on $\mathcal{R}^{(s)}$ and then applied unchanged at inference, so the wrapper performs no test time gradient updates. The frozen backbone consumes the joint multi-resolution view $\mathbf{X}_t^{(s)}$; for backbones whose interface accepts only a single stream (e.g., target history only or covariate only models), the unused component is simply omitted without otherwise changing the procedure.

\paragraph{Adaptive Ensemble Fusion} We choose weights that emphasize strides with low error. Our adaptive fusion strategy uses validation performance as a lightweight proxy for per stride reliability. Let $R^2(s)$ and $\mathrm{MSE}^{(s)}$ denote the validation coefficient of determination and mean squared error for stride $s$, and let $\varepsilon>0$ be a small stabilizer. The choice of $R^2$-based weighting reflects that each stride's contribution should be proportional to its explained variance, naturally emphasizing strides that better capture the underlying signal structure. When systematic biases dominate across all strides, every $R^2(s)$ becomes non positive (each stride does worse than predicting the mean), so the $R^2$-based weights are undefined or negative. We therefore fall back to inverse MSE weighting, which is justified because it remains strictly positive and normalisable, still assigns larger weight to lower error strides, and recovers a sensible ranking precisely in the regime where $R^2$ is uninformative:
\begin{equation}
\label{eq:r2weights}
w^{(s)} = \begin{cases}
\dfrac{\max(R^2(s),0)}{\sum_{s' \in \mathcal{S}} \max(R^2(s'),0)} & \text{if } \exists\, s: R^2(s) > 0 \\[8pt]
\dfrac{1/(\mathrm{MSE}^{(s)} + \varepsilon)}{\sum_{s' \in \mathcal{S}} 1/(\mathrm{MSE}^{(s')} + \varepsilon)} & \text{otherwise}
\end{cases}
\end{equation}
The $\max(\cdot,0)$ truncation guarantees that every weight is non negative and that the weights form a valid simplex, so no stride can contribute a negative coefficient.
The final ensemble output then uses these adaptive weights in Eq.~\eqref{eq:smr2_forecast}.

\subsection{MBB Wrapper}

Foundation model forecasts can be sensitive to the exact target history
supplied at inference. MBB reduces this sensitivity by constructing
\(K\) temporally coherent perturbations of the target history, feeding
each perturbed history into the same frozen backbone, and averaging
the resulting forecasts. The covariate history is kept unchanged; only
the target history is perturbed.

Figure~\ref{fig:mbb_motivation} illustrates this perturb and average
workflow: block-resampled residual sequences generate multiple perturbed
target histories, each perturbed history induces a frozen-backbone forecast,
and their average forms the final MBB forecast.

The perturbations are generated from validation residuals of the frozen
backbone. Specifically, we first collect residuals
\(\{r_t = y_t-\hat{y}_t\}_{t\in\mathcal{T}_{\mathrm{val}}}\), where
\(\hat{y}_t\) is the corresponding one-step forecast from the frozen
backbone. Since these residuals are serially correlated, sampling
individual residuals independently would break local autocorrelation.
MBB therefore samples short consecutive residual blocks rather than
individual residuals. Sampling and concatenating such blocks produces
residual perturbation sequences that preserve local temporal dependence
while generating diverse target-history variants.

\begin{figure}[t]
   \centering
   \includegraphics[width=1\linewidth]{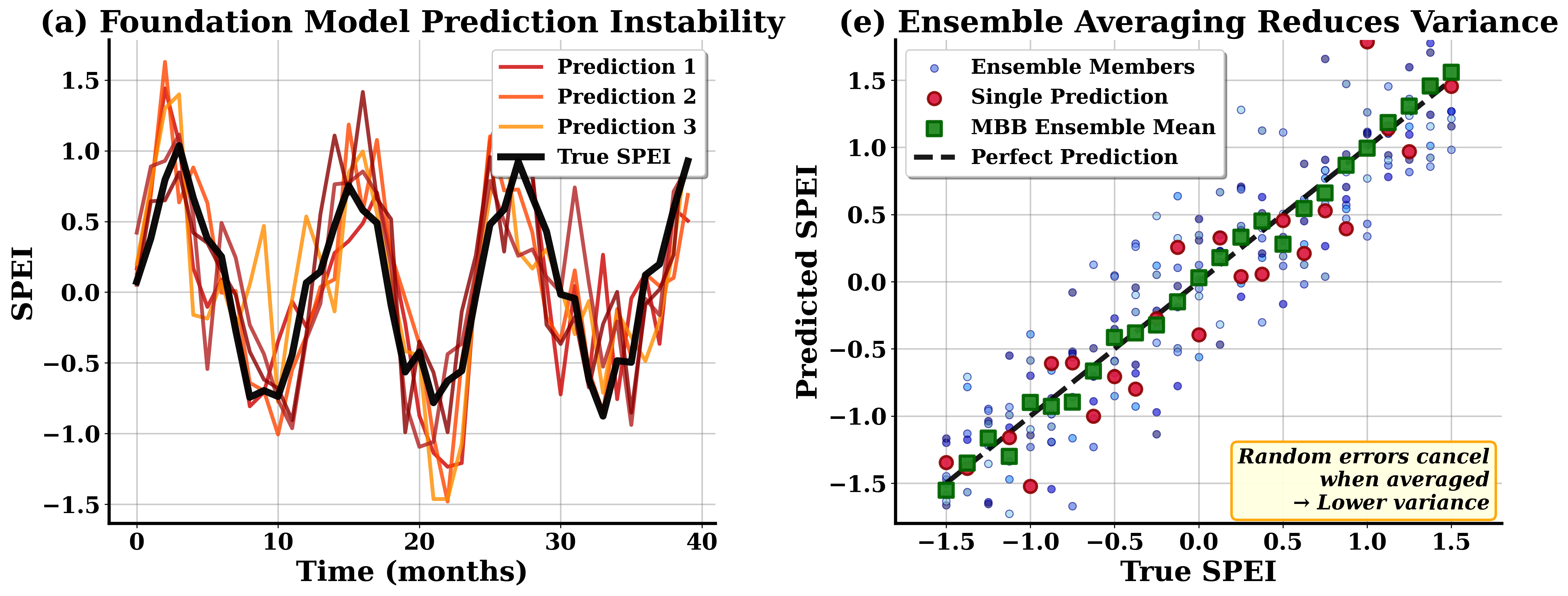}
   \caption{Block bootstrap ensembling motivation. (a) Input perturbations produce a family of frozen backbone forecasts. (b) MBB averages forecasts from block resampled residual histories.}
   \label{fig:mbb_motivation}
\end{figure}

\begin{figure*}[htbp]
  \centering
  \begin{subfigure}[b]{0.25\textwidth}
    \includegraphics[width=\textwidth]{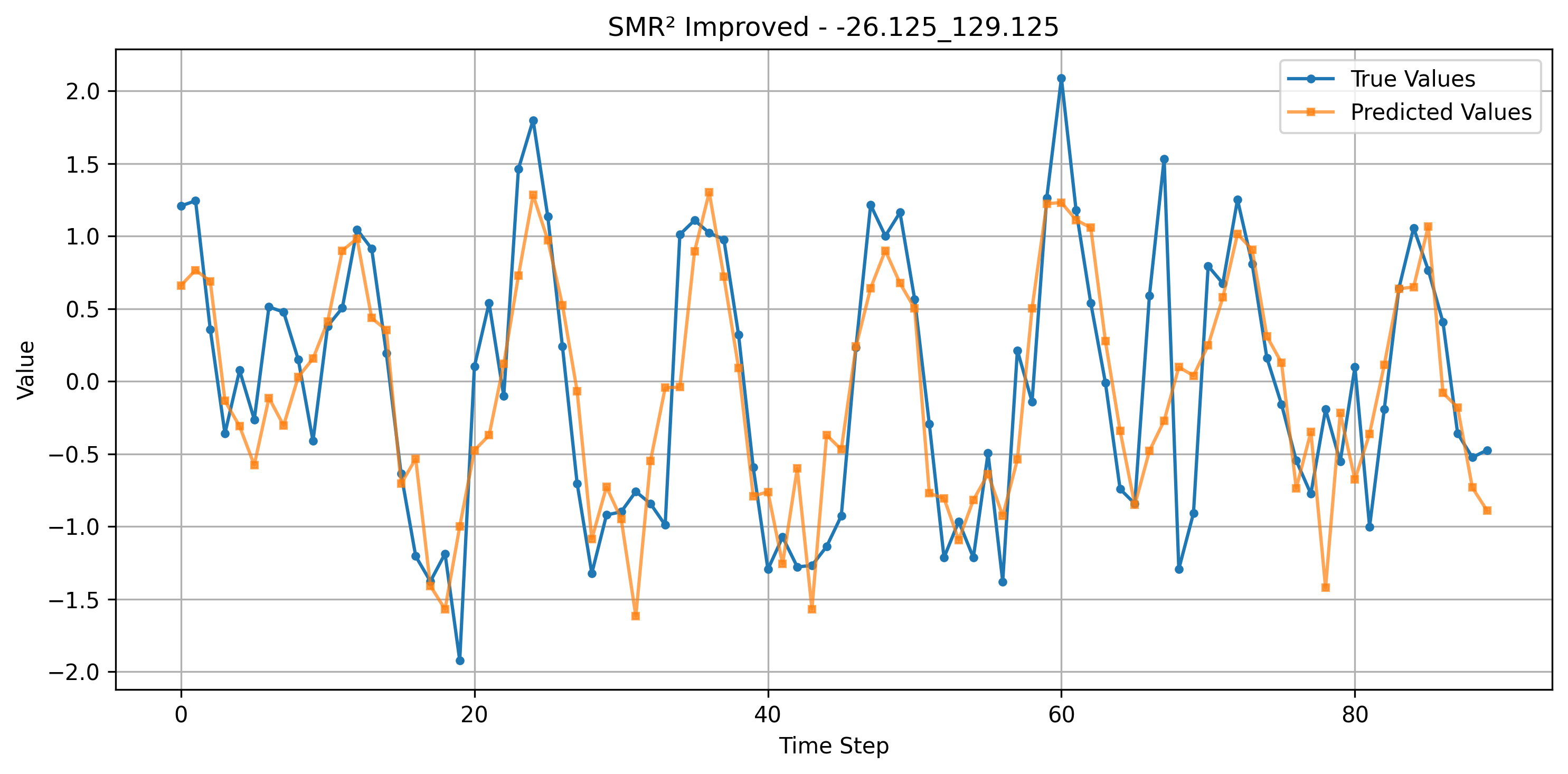}
    \caption{SMR\textsuperscript{2} (Our Method)}
    \label{fig:smr2_our}
  \end{subfigure}%
  \begin{subfigure}[b]{0.25\textwidth}
    \includegraphics[width=\textwidth]{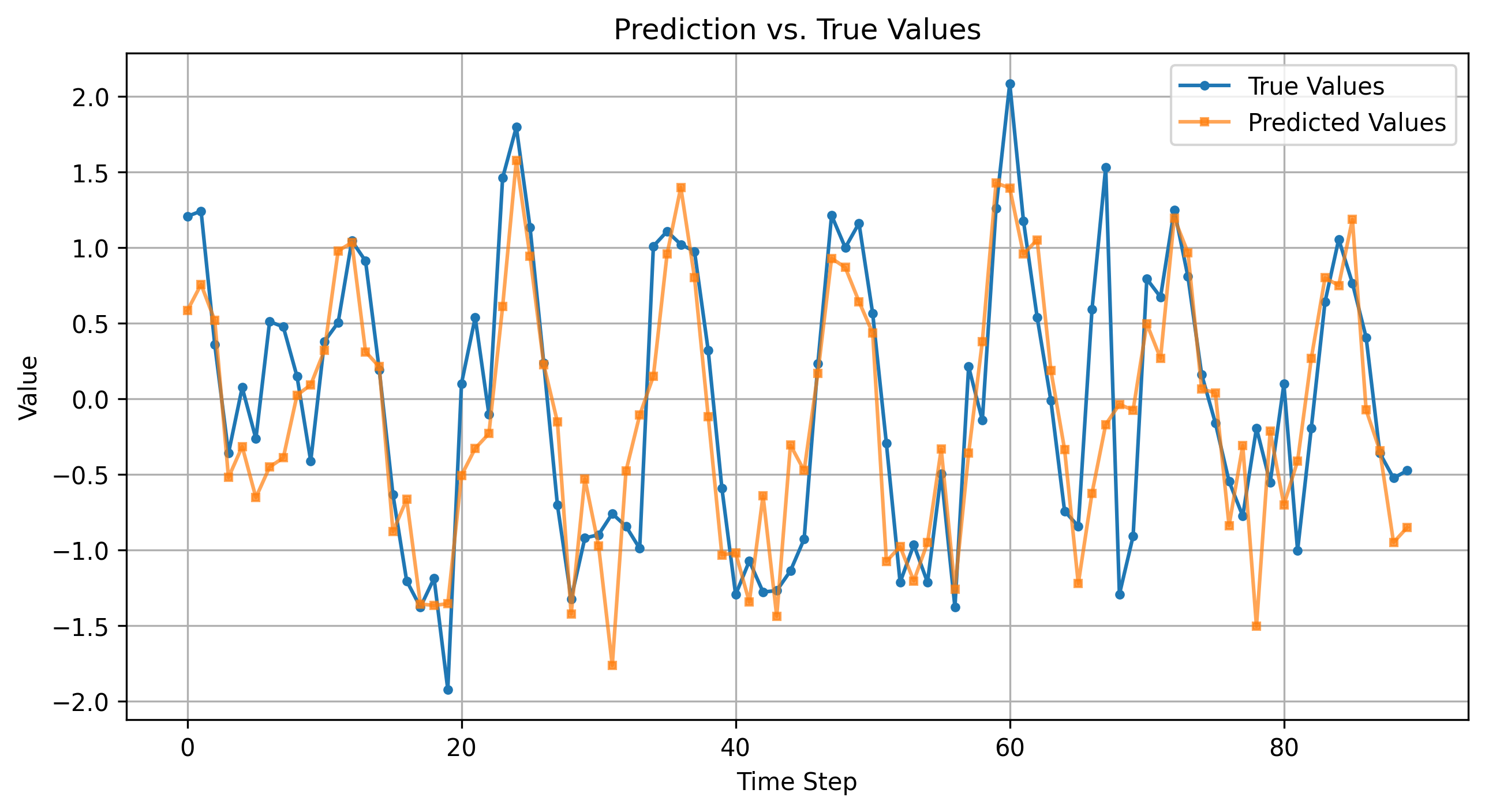}
    \caption{TimesFM}
    \label{fig:transformer}
  \end{subfigure}%
  \begin{subfigure}[b]{0.25\textwidth}
    \includegraphics[width=\textwidth]{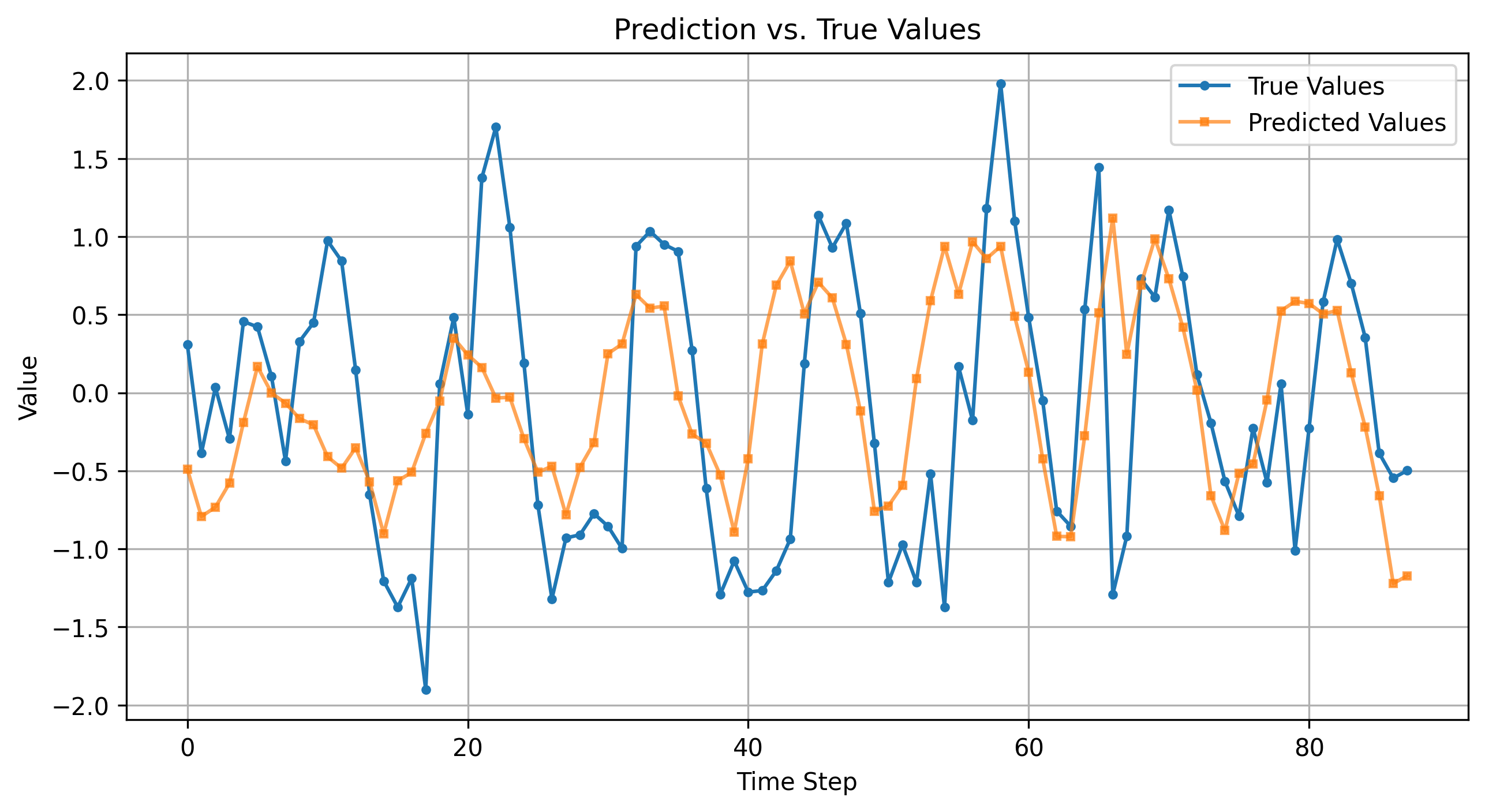}
    \caption{Autoformer}
    \label{fig:autoformer}
  \end{subfigure}%
  \begin{subfigure}[b]{0.25\textwidth}
    \includegraphics[width=\textwidth]{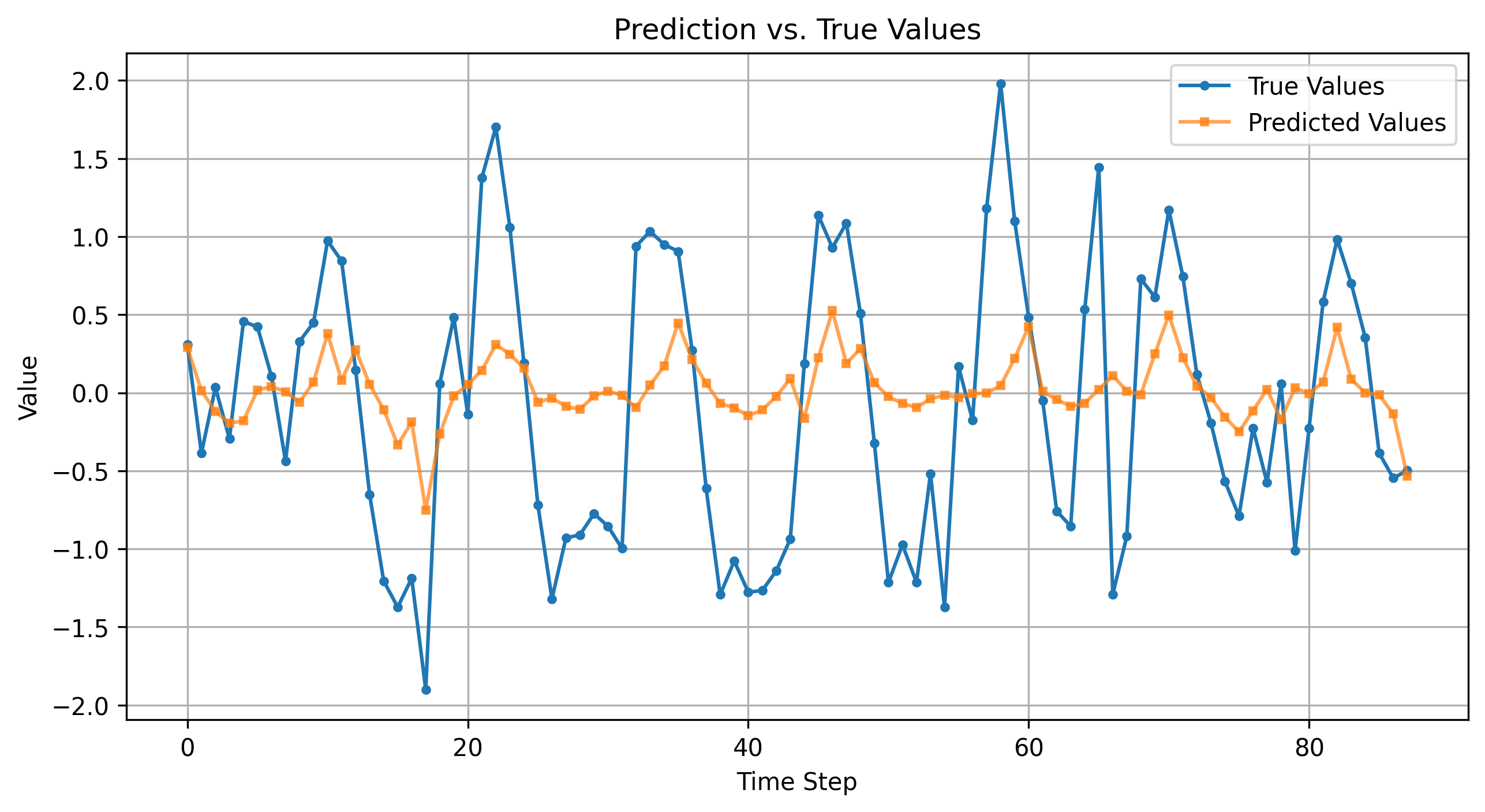}
    \caption{Crossformer}
    \label{fig:crossformer}
  \end{subfigure}%
  \\
  \begin{subfigure}[b]{0.25\textwidth}
    \includegraphics[width=\textwidth]{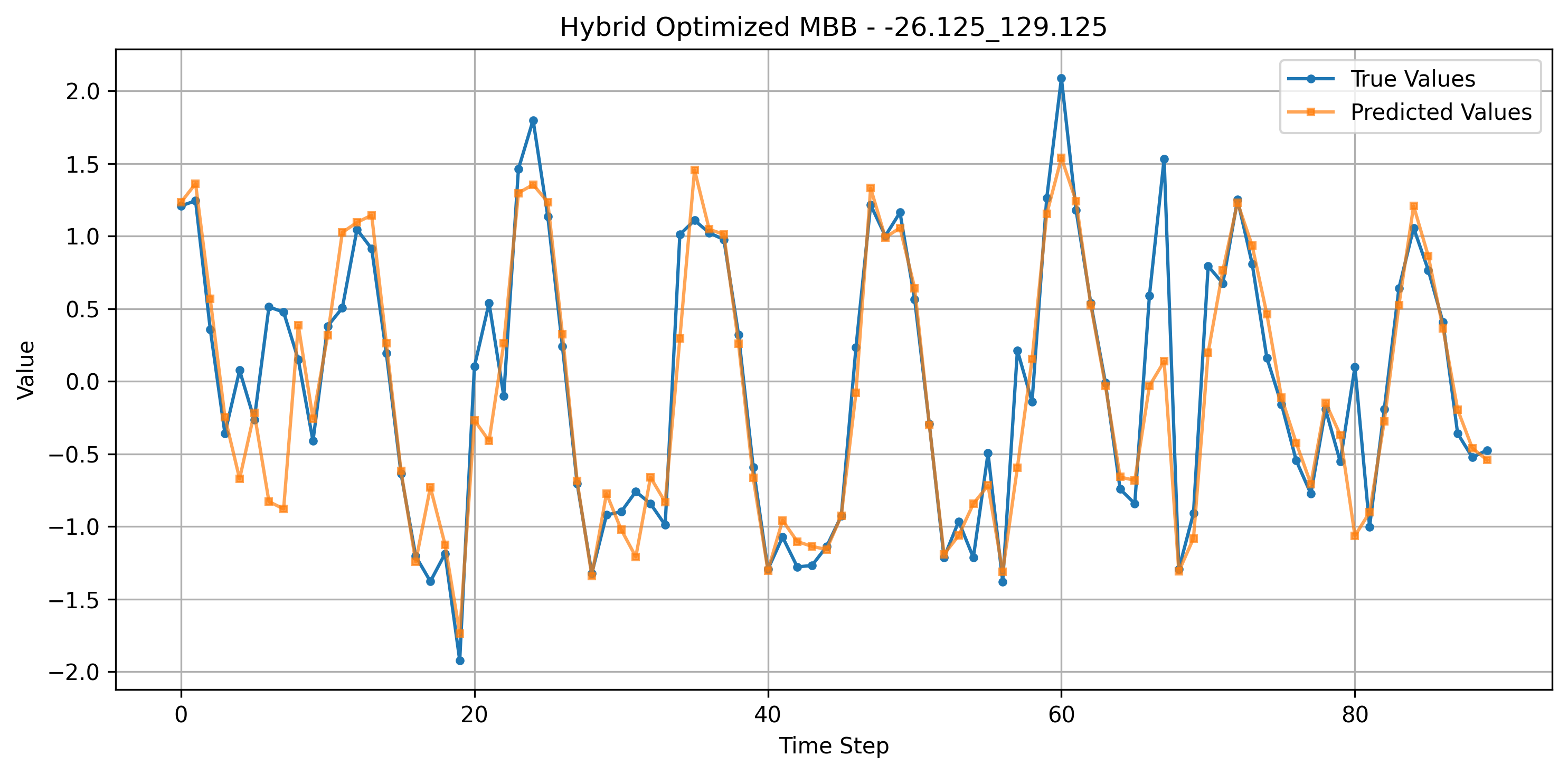}
    \caption{MBB (Our Method)}
    \label{fig:mbb_our}
  \end{subfigure}%
  \begin{subfigure}[b]{0.25\textwidth}
    \includegraphics[width=\textwidth]{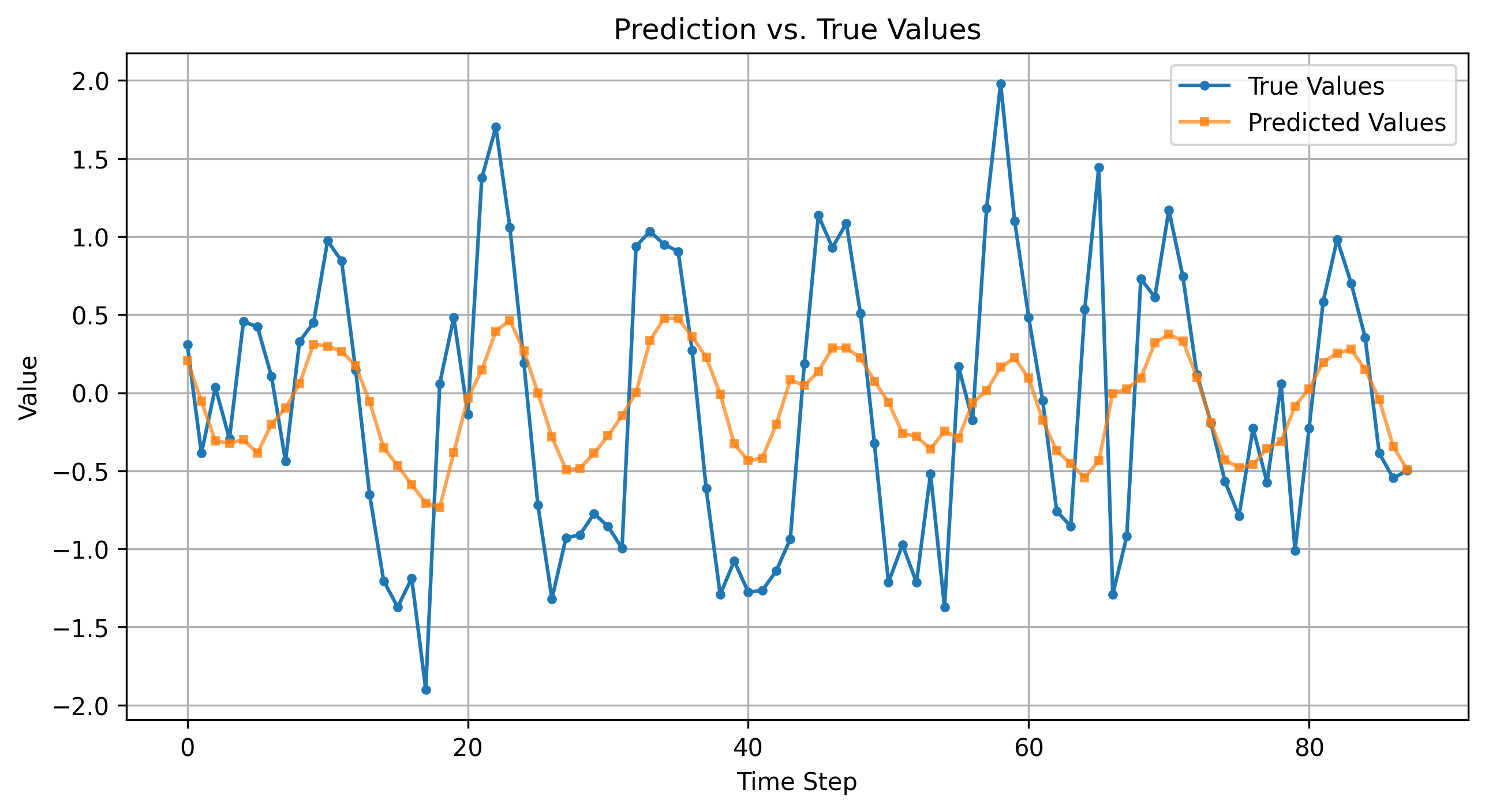}
    \caption{DLinear}
    \label{fig:dlinear}
  \end{subfigure}%
  \begin{subfigure}[b]{0.25\textwidth}
    \includegraphics[width=\textwidth]{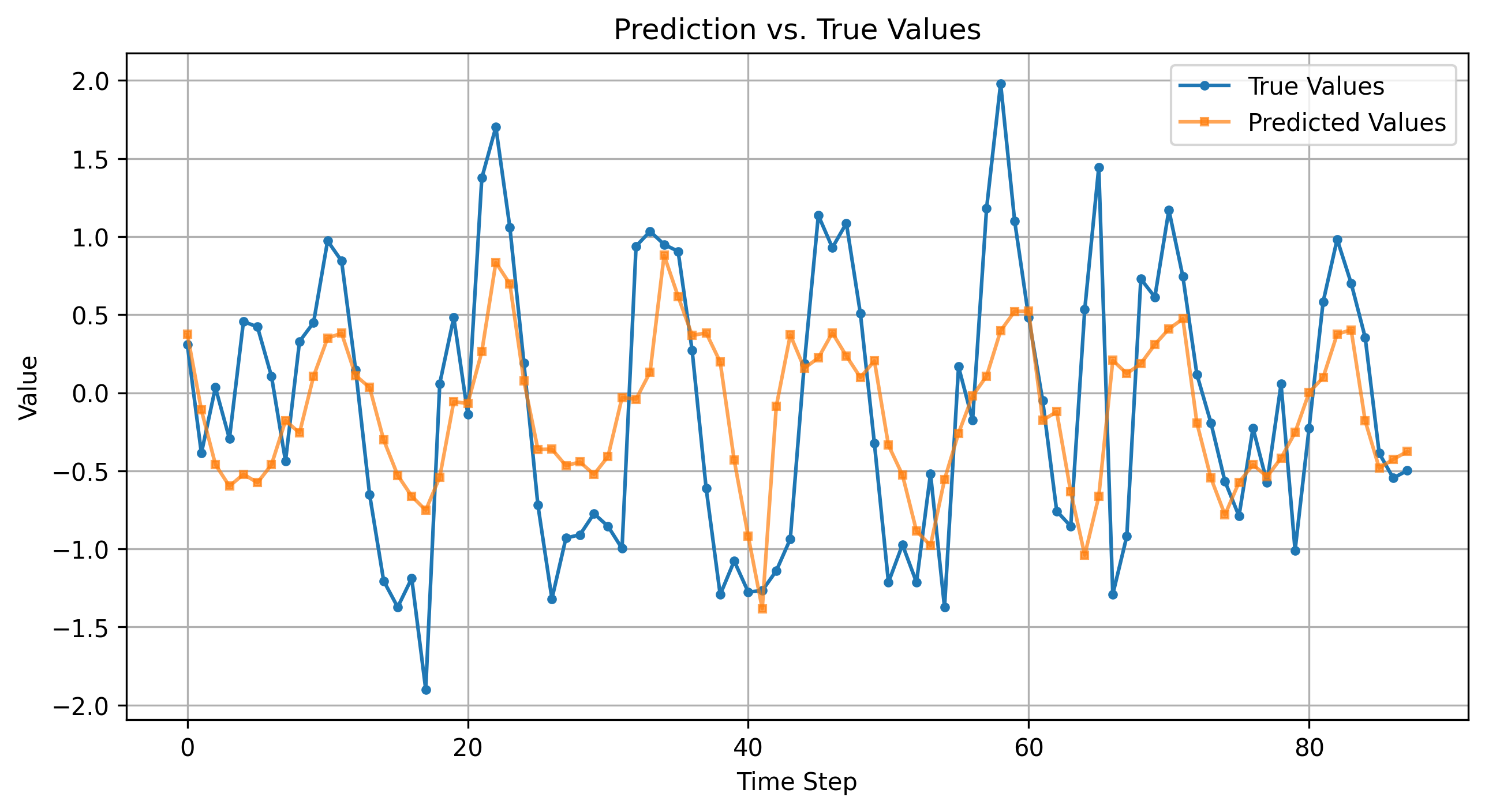}
    \caption{iTransformer}
    \label{fig:itransformer}
  \end{subfigure}%
  \begin{subfigure}[b]{0.25\textwidth}
    \includegraphics[width=\textwidth]{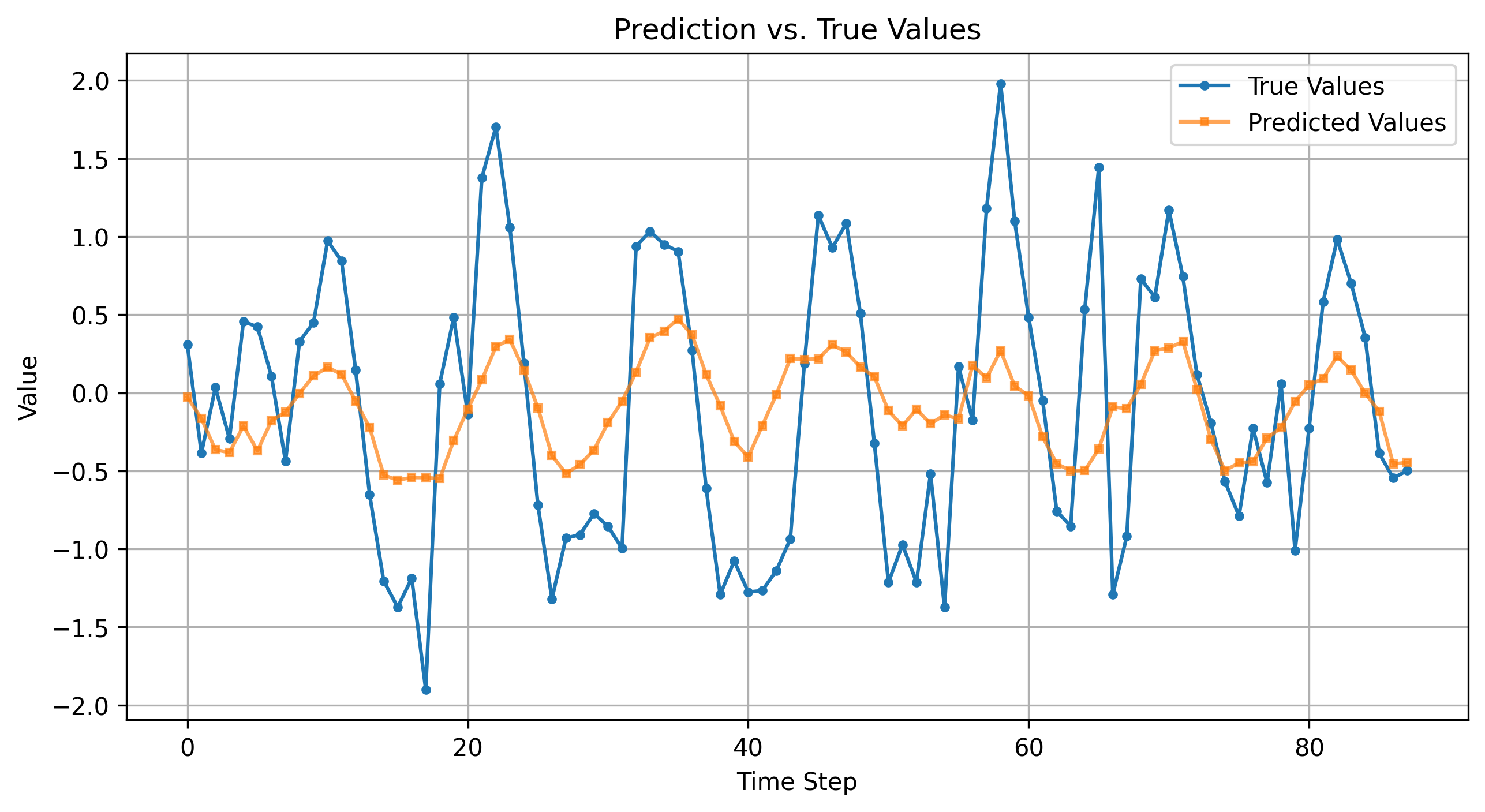}
    \caption{PatchTST}
    \label{fig:patchtst}
  \end{subfigure}%
  \caption{Predicted and observed SPEI time series for one month ahead forecasting at Location 1. Panels (a) and (e) show our inference time wrappers SMR\textsuperscript{2} and MBB, while panels (b to d, f to h) present representative baselines. The wrapper panels visually track the observed series more closely for this representative site, consistent with the aggregate gains reported in Table~\ref{tab:main_results}.}
  \label{fig:wrapper_vs_baselines}
\end{figure*}

\paragraph{Residual Block Construction}

Following standard MBB practice~\cite{kunsch1989jackknife,Lahiri1993}, we set the block length to $b = \lfloor n^{1/3} \rfloor$ where $n = |\mathcal{T}_{\text{val}}|$, a common asymptotic rule of thumb under standard dependence assumptions. We construct overlapping candidate blocks:
\begin{equation}
B_i = \{r_i, r_{i+1}, \ldots, r_{i+b-1}\}, \quad i \in \{1, \ldots, n-b+1\}
\end{equation}
This overlapping construction lets the finite validation residual series contribute many possible local error patterns.

\paragraph{Bootstrap Resampling Strategy}
To generate a perturbed sequence of length $L$, we sample $\lceil L/b \rceil$ blocks with replacement. The sampling process creates the ordered block sequence
\begin{equation}
\mathbf{B}^{(k)} = (B_1^{(k)}, B_2^{(k)}, \ldots, B_{\lceil L/b \rceil}^{(k)})
\end{equation}
where each $B_j^{(k)}$ is randomly selected from the available blocks. These blocks are then concatenated and truncated to form residual sequences of the required length:
\begin{equation}
r_{1:L}^{(k)} = [B_1^{(k)}, B_2^{(k)}, \ldots]_{1:L}
\end{equation}
where $[\,\cdot\,]$ denotes the same vector concatenation used earlier (e.g., in the multi-resolution input construction) and $[\,\cdot\,]_{1:L}$ truncates the concatenated blocks to length $L$. This process generates $K$ different residual sequences used to form the perturbed histories below.

\paragraph{Ensemble Prediction}
For each bootstrap sample $k$, we construct perturbed target histories by adding the position indexed resampled residual sequence to the original predictions over the history window:
\begin{equation}
\tilde{y}_{t-L+i}^{(k)} = \hat{y}_{t-L+i} + r_i^{(k)}, \quad i = 1, \ldots, L
\end{equation}
Stacking these over the window gives the perturbed target history $\tilde{y}_{t-L+1:t}^{(k)} = (\tilde{y}_{t-L+1}^{(k)}, \ldots, \tilde{y}_{t}^{(k)})$, that is, the foundation model's own predicted history corrected by the $k$th resampled residual sequence. Here $r_i^{(k)}$ is the $i$th element of $r_{1:L}^{(k)}$, and $\hat{y}_{t-L+i}$ are the frozen backbone's rolling one step predictions over the history window, obtained under the same one step protocol as inference. Pairing it with the covariates gives the $k$th perturbed input $\tilde{\mathbf{X}}_t^{(k)} = (\mathbf{c}_{t-L+1:t}, \tilde{y}_{t-L+1:t}^{(k)})$, which is fed to the frozen foundation model:
\begin{align}
\hat{y}_{t+1}^{(k)} &= f_{\theta}^{\text{base}}(\tilde{\mathbf{X}}_t^{(k)}) \\
\hat{y}_{t+1}^{\mathrm{MBB}} &= \frac{1}{K} \sum_{k=1}^{K} \hat{y}_{t+1}^{(k)}
\end{align}
The second line is the ensemble point forecast; the variance argument for this average is stated in the theoretical analysis below.

\subsection{Theoretical Analysis}
\label{sec:theoretical_analysis}

Because the wrappers are deliberately lightweight, the analysis below is to make explicit \emph{why} a black-box correction of a frozen TSFM should help: why the residual of a frozen backbone is learnable, why combining multiple temporal resolutions is sound, and why a block bootstrap ensemble stabilizes the point forecast. Everything rests on a single assumption and two short, assumption light results; the derivations are in Appendix~A.

\paragraph{Residual notation in the analysis.}
We use different residual symbols because the two wrappers use
residuals for different purposes. In SMR$^2$, $\epsilon_t^{(s)}$
denotes the reduced-context, stride-specific calibration residual
defined in Eq.~(5). It is used to fit the stride-specific residual
corrector and to produce the learned correction
$\hat{\epsilon}_{t+1}^{(s)}$ in Eq.~(7). In MBB, $r_t$ denotes the
full-context validation residual of the frozen backbone, used only
for constructing temporally coherent block-bootstrap perturbations.
Thus, $\epsilon_t^{(s)}$ is a supervised correction target, whereas
$r_t$ is a resampling object.

\begin{assumption}[Multi-resolution residual decomposition]
\label{ass:residual}
Let
\[
r_{t+1}
=
y_{t+1}
-
f_{\theta}^{\mathrm{base}}(\mathbf{X}_t)
\]
be the full-context one-step residual of the frozen backbone. We
assume that its predictable part can be decomposed into
stride-specific, regionally predictable components plus
unpredictable noise,
\begin{equation}
\label{eq:residual_decomp}
r_{t+1}
=
\sum_{s \in \mathcal{S}} r_{t+1}^{(s)}
+
\eta_{t+1},
\end{equation}
where $r_{t+1}^{(s)}$ is a latent theoretical component of the
full-context backbone residual that is visible at temporal stride
$s$, and $\eta_{t+1}$ is irreducible noise. The term
$r_{t+1}^{(s)}$ is not an observed residual sequence; SMR$^2$
estimates its predictable effect through the learned correction
$\hat{\epsilon}_{t+1}^{(s)}$.
\end{assumption}

This is consistent with what Figures~\ref{fig:enter_label} and~\ref{fig:smr2_motivation} display: every backbone in Figure~\ref{fig:enter_label} (SARIMA, iTransformer, TimesFM) leaves a \emph{systematic}, non random residual, and Figure~\ref{fig:smr2_motivation} shows the residual distribution changing shape across strides, so the multi-resolution components carry complementary information rather than copies of the same signal.

\paragraph{SMR\textsuperscript{2}: a convex risk bound.}
Recall the per stride corrected forecast $\tilde{y}_{t+1}^{(s)} = \hat{y}_{t+1}^{(s)} + \hat{\epsilon}_{t+1}^{(s)}$ and the ensemble $\tilde{y}_{t+1}^{\mathrm{ens}} = \sum_{s\in\mathcal{S}} w^{(s)}\, \tilde{y}_{t+1}^{(s)}$ from the SMR\textsuperscript{2} wrapper above, with non negative weights obeying $\sum_{s\in\mathcal{S}} w^{(s)} = 1$.

\begin{proposition}[Ensemble risk bound]
\label{prop:smr2}
Writing $e^{(s)} = \tilde{y}_{t+1}^{(s)} - y_{t+1}$ for the per stride error, $\mathbf{w}=(w^{(s)})_{s\in\mathcal{S}}$ for the ensemble weight vector, $\boldsymbol{\beta} = (\mathbb{E}[e^{(s)}])_{s\in\mathcal{S}}$ for the bias vector, $\boldsymbol{\Sigma} = \mathrm{Cov}\big((e^{(s)})_{s\in\mathcal{S}}\big)$ for the error covariance, and $\mathcal{L}_{\mathrm{ens}}$ for the ensemble mean squared error, we have
\begin{equation}
\label{eq:jensen}
\mathbb{E}\big[(\tilde{y}_{t+1}^{\mathrm{ens}} - y_{t+1})^2\big] \;\le\; \sum_{s\in\mathcal{S}} w^{(s)}\, \mathbb{E}\big[(e^{(s)})^2\big],
\end{equation}
and decomposes exactly as
\begin{equation}
\label{eq:biasvar}
\mathcal{L}_{\mathrm{ens}} = (\mathbf{w}^{\!\top}\boldsymbol{\beta})^2 + \mathbf{w}^{\!\top}\boldsymbol{\Sigma}\,\mathbf{w}.
\end{equation}
\end{proposition}

Inequality~\eqref{eq:jensen} (a direct consequence of Jensen's inequality, since the weights lie on the simplex) guarantees the ensemble risk is controlled by the weighted average of the per stride risks. The decomposition~\eqref{eq:biasvar} motivates allocating mass to low error strides while accounting for inter stride dependence: when the covariance structure is favourable, mixing weakly correlated reliable strides can reduce the variance term $\mathbf{w}^{\!\top}\boldsymbol{\Sigma}\,\mathbf{w}$ relative to a single chosen stride. The $R^2$/inverse MSE weights in Eq.~\eqref{eq:r2weights} serve as a lightweight validation proxy for this allocation by favouring low risk resolutions, while remaining covariance agnostic. Intuitively, coarser strides average out high frequency fluctuations and expose slower regional trends, so different resolutions correct different parts of the residual in Eq.~\eqref{eq:residual_decomp}.

\paragraph{MBB: variance stabilization.}
The MBB wrapper averages $K$ forecasts obtained from block bootstrapped residual perturbations of the target history. For any such ensemble,
\begin{equation}
\label{eq:mbbvar}
\mathrm{Var}\!\left(\tfrac{1}{K}\sum_{k=1}^{K}\hat{y}_{t+1}^{(k)}\right) = \frac{1}{K^2}\sum_{i,j}\mathrm{Cov}\big(\hat{y}_{t+1}^{(i)}, \hat{y}_{t+1}^{(j)}\big),
\end{equation}
which equals $\tfrac{1}{K}\mathrm{Var}(\hat{y}_{t+1})$ when the bootstrap forecasts are independent and, for approximately exchangeable forecasts of comparable variance, falls below the single forecast variance whenever their pairwise correlation is below one. Under standard stationarity and mixing conditions the Moving Block Bootstrap is a consistent resampling scheme for dependent series~\cite{kunsch1989jackknife,Lahiri1993,Politis1994}; preserving within block autocorrelation (block length $b=\lfloor n^{1/3}\rfloor$) keeps the perturbed histories realistic, so averaging over them stabilizes the point forecast without touching the backbone.

\section{Experiments}

\subsection{Study Area}
We focus our evaluation on South Australia, shown in Figure~\ref{fig:studyArea}, a region characterized by high drought vulnerability and complex climate variability driven by multiple interacting modes (ENSO, IOD). South Australia occupies a unique position in Australia's climate system, experiencing a temperate climate with dry summers in the south transitioning to a persistently dry hot desert climate in the north \citep{rashid2019characterization}. The region's climate exhibits high spatial and temporal variability, with documented warming trends since the 1970s and strong dependence on multiple large scale climate drivers. It provides high quality long term meteorological data and represents a critical water security area where improved drought forecasting has significant practical value, while recent intensifying drought conditions create challenging adaptation scenarios for globally pre trained foundation models. \textit{All datasets used in this study are publicly available and reproducible through standardized preprocessing pipelines, enabling transparent evaluation and potential deployment across other semi arid regions.}

\begin{figure}[h]
    \centering
    \includegraphics[width=0.9\linewidth]{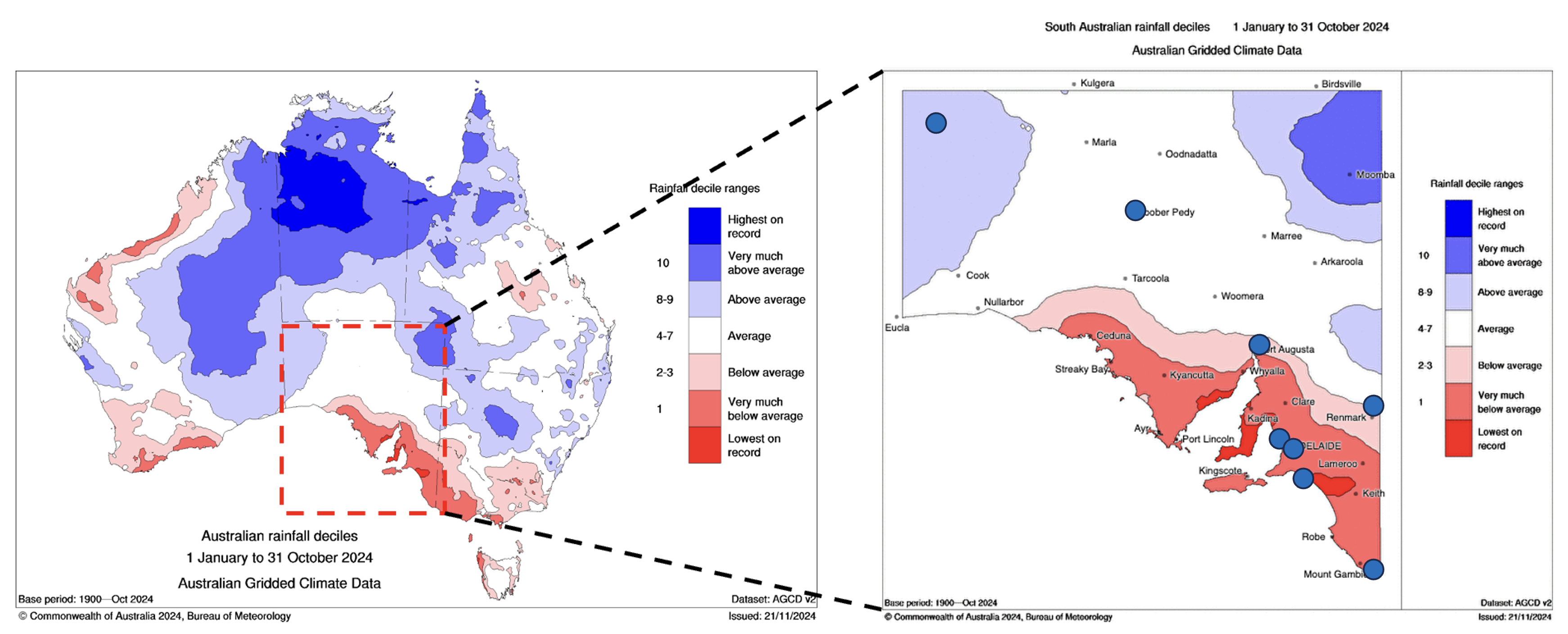}
    \caption{Study Area in South Australia}
    \label{fig:studyArea}
\end{figure}

\subsection{Experimental Setup}
\textbf{Dataset.} We evaluate on monthly SPEI-30 prediction across 3 South Australian locations over 1982--2018. The target is the SPEI-30 variable from the global daily SPEI-GD database \cite{liu2024speigd} (0.25$^\circ$ spatial resolution, natively spanning 1982--2021), converted to a monthly cadence and restricted to 1982--2018. The dataset comprises 39 variables: the target SPEI plus $d=38$ predictors, namely 34 meteorological variables from the NCEP--NCAR Reanalysis~1 \cite{kalnay1996ncep} and 4 climate indices (e.g., ENSO and IOD \cite{trenberth1997enso,saji1999dipole}). We use a chronological 70/10/20 train/validation/test split; each forecast uses a context window of length $L$ covering the available training history.

\textbf{Baselines.} We compare our wrappers against three categories of models: (i) classical statistical and recurrent baselines (ARIMA \cite{Box2015TimeSeriesAnalysis}, ETS, LSTM \cite{Hochreiter1997}, GRU, and TCN \cite{Bai2018}); (ii) Transformer based architectures including standard Transformer \cite{vaswani2017attention}, Autoformer \cite{Wu2021}, Crossformer \cite{Zhang2022}, DLinear \cite{Zeng2023}, FiLM \cite{Dongsheng2022}, iTransformer \cite{liu2024itransformerinvertedtransformerseffective}, PatchTST \cite{Nie2022}, TimesNet \cite{wu2023timesnet_iclr}, and TSMixer \cite{Chen2023}; and (iii) foundation models including TabPFN \cite{hoo2025tablestimetabpfnv2outperforms}, TimeGPT \cite{garza2023timegpt}, and TimesFM \cite{Sen2024}. All models are evaluated using identical data splits and preprocessing pipelines. 

\textbf{Implementation Details.} SMR\textsuperscript{2} uses strides $\mathcal{S} = \{1, 2, 4, 6, 8\}$ selected to capture multi-resolution temporal patterns from monthly to seasonal dynamics (see Appendix~B for details). The ridge residual corrector $f^{\text{res}}$ uses penalty $\lambda_{\text{ridge}} = 0.1$, residual lag length $q = 6$, and stabilizer $\varepsilon = 10^{-6}$. MBB sets block length $b = \lfloor n^{1/3} \rfloor$ and $K = 100$ bootstrap samples. Both wrappers add only inference time computation; runtime, memory, hardware, and software details are reported in Appendices~E--F.

\textbf{Evaluation Metrics.} We report Mean Squared Error (MSE), Mean Absolute Error (MAE), and coefficient of determination (R$^2$); the evaluation protocol is detailed in Appendix~D. 

\begin{table*}[h]
\centering
\renewcommand{\arraystretch}{0.85}
\scriptsize
\resizebox{\textwidth}{!}{
\begin{tabular}{l|ccc|ccc|ccc}
\toprule
\multirow{2}{*}{Method} & \multicolumn{3}{c|}{Location 1 (-26.125\_129.125)} & \multicolumn{3}{c|}{Location 2 (-29.125\_134.875)} & \multicolumn{3}{c}{Location 3 (-35.625\_138.875)} \\
& MSE$\downarrow$ & MAE$\downarrow$ & R$^2\uparrow$ & MSE$\downarrow$ & MAE$\downarrow$ & R$^2\uparrow$ & MSE$\downarrow$ & MAE$\downarrow$ & R$^2\uparrow$ \\
\midrule
\multicolumn{10}{c}{\textit{Classical and Recurrent Baselines}} \\
ARIMA & 1.213 & 0.882 & -0.417 & 1.201 & 0.875 & -0.409 & 1.225 & 0.887 & -0.426 \\
ETS & 0.950 & 0.765 & 0.050 & 0.942 & 0.760 & 0.055 & 0.957 & 0.769 & 0.048 \\
LSTM & 0.688 & 0.612 & 0.320 & 0.679 & 0.605 & 0.326 & 0.694 & 0.616 & 0.318 \\
GRU & 0.674 & 0.598 & 0.335 & 0.667 & 0.592 & 0.340 & 0.681 & 0.603 & 0.330 \\
TCN & 0.662 & 0.586 & 0.345 & 0.655 & 0.581 & 0.350 & 0.670 & 0.590 & 0.340 \\
\midrule
\multicolumn{10}{c}{\textit{Deep Learning Baselines}} \\
Transformer & 0.537 & 0.554 & 0.451 & 0.510 & 0.581 & 0.522 & 0.745 & 0.655 & 0.365 \\
Autoformer & 0.546 & 0.559 & 0.442 & 0.547 & 0.582 & 0.487 & 0.752 & 0.658 & 0.358 \\
Crossformer & 0.542 & 0.557 & 0.446 & 0.474 & 0.588 & 0.556 & 0.740 & 0.654 & 0.370 \\
DLinear & 0.533 & 0.551 & 0.455 & 0.423 & 0.553 & 0.604 & 0.735 & 0.650 & 0.375 \\
FiLM & 0.558 & 0.569 & 0.430 & 0.529 & 0.597 & 0.505 & 0.760 & 0.662 & 0.350 \\
iTransformer & 0.524 & 0.546 & 0.464 & 0.454 & 0.539 & 0.575 & 0.738 & 0.652 & 0.372 \\
PatchTST & 0.535 & 0.552 & 0.453 & 0.461 & 0.548 & 0.568 & 0.743 & 0.655 & 0.366 \\
TimesNet & 0.550 & 0.562 & 0.438 & 0.489 & 0.572 & 0.542 & 0.754 & 0.659 & 0.356 \\
TSMixer & 0.529 & 0.549 & 0.459 & 0.468 & 0.555 & 0.561 & 0.736 & 0.651 & 0.374 \\
\midrule
\multicolumn{10}{c}{\textit{Foundation Models}} \\
TabPFN-ts v1.0 & 0.462 & 0.503 & 0.521 & 0.389 & 0.445 & 0.632 & 0.566 & 0.562 & 0.442 \\
+ SMR\textsuperscript{2} (Ours) & 0.343 & 0.459 & 0.585 & 0.318 & \textbf{0.278} & 0.706 & 0.476 & 0.585 & 0.488 \\
+ MBB (Ours) & 0.351 & 0.459 & 0.609 & 0.312 & 0.407 & 0.705 & 0.459 & 0.493 & 0.534 \\
TimeGPT-1 & 0.437 & 0.489 & 0.536 & 0.351 & 0.431 & 0.658 & 0.543 & 0.551 & 0.456 \\
+ SMR\textsuperscript{2} (Ours) & 0.356 & \underline{0.440} & 0.570 & 0.280 & 0.403 & \underline{0.742} & 0.431 & \textbf{0.456} & 0.536 \\
+ MBB (Ours) & 0.331 & 0.442 & \underline{0.625} & 0.284 & 0.392 & 0.731 & 0.438 & 0.483 & 0.554 \\
TimesFM v1.0 & 0.391 & 0.465 & 0.564 & 0.318 & 0.412 & 0.687 & 0.524 & 0.536 & 0.469 \\
+ SMR\textsuperscript{2} (Ours) & \underline{0.316} & 0.446 & 0.617 & \textbf{0.251} & \underline{0.370} & \textbf{0.754} & \textbf{0.408} & \underline{0.463} & \textbf{0.582} \\
+ MBB (Ours) & \textbf{0.313} & \textbf{0.324} & \textbf{0.681} & \underline{0.265} & 0.378 & 0.735 & \underline{0.426} & 0.477 & \underline{0.564} \\
\bottomrule
\end{tabular}
}
\caption{One month ahead SPEI forecasting performance across three representative South Australian locations. Best results in \textbf{bold}, second best \underline{underlined}. Per method metrics are computed over their scored time steps (see Appendix~D).}
\label{tab:main_results}
\end{table*}

\subsection{Main Experiments}

\textbf{Forecasting Strategy.} We use a rolling one-step-ahead forecasting approach where predictions are made sequentially across the test period, with the context window sliding forward after each prediction to evaluate performance throughout the entire test set.

For baseline models, we train each architecture from scratch on the training set using their original best configurations. Traditional deep learning models (LSTM, GRU, TCN) are trained for 100 epochs with early stopping based on validation performance. Transformer based models use the same training protocol with architecture specific hyperparameters as reported in their original papers. Foundation models (TabPFN-ts V1.0, TimeGPT-1, and TimesFM v1.0) are used in their pre trained form without additional fine-tuning. This is not meant to understate their capacity: fine-tuning is precisely what the deployment setting we target rules out, since the backbone weights are often proprietary, the regional records are too short for stable fine-tuning, and operational compute budgets are limited. Evaluating the foundation models zero-shot therefore reflects the realistic regime in which our wrappers would be used, and provides a controlled comparison in which both the baselines and our method leave the backbone frozen; the wrappers add only inference time computation on top of the same frozen predictions. For each foundation model, we evaluate three configurations: baseline without modification, with SMR\textsuperscript{2} wrapper, and with MBB wrapper applied at inference time. We evaluate the two wrappers separately to isolate the effect of each; they are alternative lightweight adaptations rather than a combined pipeline.

We evaluate performance across three representative locations in South Australia that capture different climate zones. For each location, we compute all metrics separately and report location specific results in Table~\ref{tab:main_results}. Figure~\ref{fig:wrapper_vs_baselines} visualizes the predicted versus observed SPEI at Location 1 and provides representative visual evidence consistent with the quantitative improvements in Table~\ref{tab:main_results}.

Table~\ref{tab:main_results} shows that frozen foundation models already outperform most task trained baselines, while SMR\textsuperscript{2} and MBB further improve every frozen backbone across all three sites, indicating complementary regional corrections without parameter updates.

\begin{table}[h]
\centering
\setlength{\tabcolsep}{4pt}
\begin{tabular}{lccc}
\toprule
Configuration & MSE$\downarrow$ & R$^2\uparrow$ & $\Delta$MSE (\%) \\
\midrule
\multicolumn{4}{c}{\textit{Baseline}} \\
TimesFM (frozen)                & 0.391 & 0.564 & --   \\
\midrule
\multicolumn{4}{c}{\textit{SMR\textsuperscript{2} incremental switches}} \\
S-1 Single stride + Residual           & 0.383 & 0.573 & 2.0  \\
S-2 Multiple strides                  & 0.375 & 0.583 & 4.1  \\
S-3 + Adaptive $R^2$ weighting       & 0.368 & 0.592 & 5.9  \\
S-4 + Two phase calibration            & 0.355 & 0.603 & 9.2  \\
\textbf{S-Full}                        & \textbf{0.316} & \textbf{0.617} & \textbf{19.2} \\
\midrule
\multicolumn{4}{c}{\textit{MBB incremental switches}} \\
B-1 IID bootstrap (block = 1)          & 0.384 & 0.571 & 1.8  \\
B-2 Fixed block length                 & 0.379 & 0.578 & 3.1  \\
B-3 + Overlapping blocks               & 0.351 & 0.600 & 10.2 \\
\textbf{B-Full}                        & \textbf{0.313} & \textbf{0.681} & \textbf{19.9} \\
\bottomrule
\end{tabular}
\caption{Ablation study at Location 1 ($-26.125^\circ$, $129.125^\circ$).}
\label{tab:ablation_loc1}
\end{table}

\subsection{Ablation Study}

Table~\ref{tab:ablation_loc1} presents an ablation study on TimesFM v1.0 at Location 1, systematically evaluating each component's contribution. For SMR\textsuperscript{2}, incremental addition of components shows progressive improvements: single stride residual correction (S-1) reduces MSE by 2.0\%, S-2 uses multiple strides with uniform weights and achieves 4.1\% reduction, and the full configuration (S-Full), which combines adaptive $R^2$ weighting, two phase calibration, and the mean/std multi-resolution embedding, delivers 19.2\% MSE reduction. For MBB, the ablation isolates the effect of block sampling: IID bootstrap (B-1) provides minimal improvement (1.8\%), while the full MBB configuration (B-Full) with overlapping resampling and the standard block length $b=\lfloor n^{1/3}\rfloor$ achieves 19.9\% MSE reduction. Both wrappers show substantial and complementary improvements over the frozen baseline.

\subsection{Generalization to Other Regions}
To evaluate cross regional robustness, we further test our inference time wrappers with the TimesFM v1.0 backbone on three climatically distinct regions: Indonesia (tropical monsoon), North Africa (arid desert), and the U.S. West Coast (Mediterranean). These regions exhibit markedly different rainfall seasonality and drought dynamics, providing a stringent generalization test.

All regions follow the same experimental protocol as South Australia, including identical predictor target construction, preprocessing, rolling origin evaluation, and hyperparameter settings. Monthly reanalysis variables are used as predictors, with the SPEI-30 variable from the global daily SPEI-GD dataset (0.25°, converted to a monthly cadence over 1982--2018) as the target. The backbone TimesFM model remains fully frozen.

As summarized in Table \ref{tab:generalization}, SMR\textsuperscript{2} reduces MSE in every region, by roughly 16 to 41\%, while MBB also improves the frozen backbone under the same cross regional protocol. The largest relative gains appear on the U.S. West Coast and in North Africa, with a smaller but still clear improvement on the tropical monsoon Indonesia case. These results demonstrate that the proposed wrappers generalize effectively across heterogeneous climate regimes without fine-tuning or parameter updates.

\begin{table}[h]
\centering
\begin{tabular}{lccc}
\toprule
\textbf{Region} & \textbf{Method} & \textbf{MSE$\downarrow$} & \textbf{R$^2\uparrow$} \\
\midrule
\multirow{3}{*}{Indonesia} 
 & TimesFM v1.0 (frozen) & 0.260 & 0.485 \\
 & + SMR\textsuperscript{2} & 0.218 & 0.557 \\
 & + MBB & 0.229 & 0.535 \\
\midrule
\multirow{3}{*}{North Africa} 
 & TimesFM v1.0 (frozen) & 0.101 & 0.892 \\
 & + SMR\textsuperscript{2} & 0.0723 & 0.922 \\
 & + MBB & 0.0765 & 0.918 \\
\midrule
\multirow{3}{*}{U.S. West Coast} 
 & TimesFM v1.0 (frozen) & 0.155 & 0.854 \\
 & + SMR\textsuperscript{2} & 0.0921 & 0.912 \\
 & + MBB & 0.0958 & 0.907 \\
\bottomrule
\end{tabular}
\caption{
Cross regional generalization results on monthly SPEI prediction (1982--2018).
}
\label{tab:generalization}
\end{table}

\section{Social Impact}

This work supports the responsible use of pre trained time series foundation models for regional drought analysis under realistic data and deployment constraints.

Improved forecast reliability facilitates more informed interpretation of model outputs in downstream applications such as water resource management, agricultural planning, and disaster preparedness, helping practitioners assess risk under climate variability.The framework may also help extend scientific foundation models to regional settings where intrusive model modification is infeasible.

\section{Limitations, Ethical Considerations, and Conclusion}

This work evaluates monthly SPEI forecasting across multiple regions and demonstrates stable cross regional generalization within the evaluated scope. The study focuses on a limited set of temporal resolutions and climate settings, and performance under substantially different temporal granularities, climate regimes, or hydrological targets has not been systematically validated. In addition, the framework relies on sufficient historical observations for inference time calibration, and may be affected by data scarcity or by inherent biases of frozen foundation models under extreme regime shifts.

All data used are publicly available and do not involve human subjects or personal information. As drought forecasts may inform high stakes decisions such as water resource management and agricultural planning, predictions should be used as decision support in conjunction with local observations and domain expertise.

In summary, we proposed a lightweight inference time adaptation framework for applying frozen time series foundation models to regional drought forecasting. Without modifying model parameters or training objectives, the SMR\textsuperscript{2} and MBB wrappers effectively improve predictive accuracy in data scarce regional climate settings. Owing to its modular, backbone agnostic design, the framework is broadly applicable to climate science tasks that rely on frozen or black-box models.

\section*{GenAI Disclosure}
We used a large language model (LLM) solely as a writing assistant to translate and improve grammar, fluency, and consistency of the manuscript text. All ideas, methods, analyses, and conclusions are by the authors; all citations and numerical values were manually verified against primary sources. No non-public data were provided to the tool.

\section*{Acknowledgments}
This work was supported by the ARC Discovery Project DP230101122 and the Adelaide University Research Training Program (RTP) Scholarship. We gratefully acknowledge the continued support from the CSIRO Environment and Technology Research Units.

\bibliographystyle{ACM-Reference-Format}
\balance
\bibliography{references}


\appendix

\newpage
\section*{Appendix A: Residual Decomposition and Risk Bound}

This appendix states the assumption and the two short results used in the main text, and proves the only non trivial one (the SMR\textsuperscript{2} risk bound). The Moving Block Bootstrap properties are classical and are cited rather than reproved.

\medskip
\noindent\textbf{Assumption~\ref{ass:residual} (multi-resolution residual decomposition).}
\emph{The one step backbone residual $r_{t+1}=y_{t+1}-f_{\theta}^{\text{base}}(\mathbf{X}_t)$ decomposes as $r_{t+1}=\sum_{s\in\mathcal{S}} r_{t+1}^{(s)}+\eta_{t+1}$, with stride specific predictable components $r_{t+1}^{(s)}$ and irreducible noise $\eta_{t+1}$.}

\medskip
\noindent\textbf{Proposition~\ref{prop:smr2} (Ensemble risk bound).}
\emph{Let $e^{(s)}=\tilde{y}_{t+1}^{(s)}-y_{t+1}$ and let $\mathbf{w}=(w^{(s)})_{s\in\mathcal{S}}$ satisfy $w^{(s)}\ge0$, $\sum_s w^{(s)}=1$. Then}
\[
\begin{gathered}
\mathbb{E}\big[(\tilde{y}_{t+1}^{\mathrm{ens}}-y_{t+1})^2\big]\le \sum_{s\in\mathcal{S}} w^{(s)}\,\mathbb{E}\big[(e^{(s)})^2\big],\\[2pt]
\mathcal{L}_{\mathrm{ens}}=(\mathbf{w}^{\!\top}\boldsymbol{\beta})^2+\mathbf{w}^{\!\top}\boldsymbol{\Sigma}\,\mathbf{w}.
\end{gathered}
\]
\emph{where $\boldsymbol{\beta}=(\mathbb{E}[e^{(s)}])_{s}$ and $\boldsymbol{\Sigma}=\mathrm{Cov}\big((e^{(s)})_{s}\big)$.}

\noindent\textbf{Proof.}
Because the weights lie on the simplex, the ensemble error $\tilde{y}_{t+1}^{\mathrm{ens}}-y_{t+1}=\sum_{s} w^{(s)} e^{(s)}=\mathbf{w}^{\!\top}\mathbf{e}$ is a convex combination of the per stride errors $\mathbf{e}=(e^{(s)})_{s}$. Since $x\mapsto x^2$ is convex, Jensen's inequality gives $\big(\sum_{s} w^{(s)} e^{(s)}\big)^2\le \sum_{s} w^{(s)} (e^{(s)})^2$ pointwise; taking expectations yields the first inequality. For the identity,
\[
\mathcal{L}_{\mathrm{ens}}=\mathbb{E}\big[(\mathbf{w}^{\!\top}\mathbf{e})^2\big]=\big(\mathbb{E}[\mathbf{w}^{\!\top}\mathbf{e}]\big)^2+\mathrm{Var}(\mathbf{w}^{\!\top}\mathbf{e})=(\mathbf{w}^{\!\top}\boldsymbol{\beta})^2+\mathbf{w}^{\!\top}\boldsymbol{\Sigma}\,\mathbf{w}.
\]
When the covariance structure is favourable, for example when reliable low variance strides have sufficiently weak positive or negative correlation, suitable simplex weights can satisfy $\mathbf{w}^{\!\top}\boldsymbol{\Sigma}\,\mathbf{w}<\min_{s\in\mathcal{A}}\Sigma_{ss}$ for a subset $\mathcal{A}$ of reliable strides. In other cases, the Jensen bound above still guarantees that the ensemble risk is controlled by the weighted average of the per stride risks. \hfill$\square$

\medskip
\noindent\textbf{MBB variance stabilization (standard result).}
The MBB forecast is the average $\hat{y}_{t+1}^{\mathrm{MBB}}=\tfrac{1}{K}\sum_{k=1}^{K}\hat{y}_{t+1}^{(k)}$ of forecasts from block bootstrapped residual perturbations, with variance given by the elementary identity
\[
\mathrm{Var}\big(\hat{y}_{t+1}^{\mathrm{MBB}}\big)=\frac{1}{K^2}\sum_{i,j}\mathrm{Cov}\big(\hat{y}_{t+1}^{(i)},\hat{y}_{t+1}^{(j)}\big),
\]
which equals $\tfrac{1}{K}\mathrm{Var}(\hat{y}_{t+1})$ for independent replicates and, for approximately exchangeable replicates of comparable variance, stays below the single forecast variance whenever the pairwise correlation is below one. Consistency of the Moving Block Bootstrap for stationary, mixing sequences is a classical result of K{\"u}nsch~\cite{kunsch1989jackknife}, Lahiri~\cite{Lahiri1993}, and Politis and Romano~\cite{Politis1994}; block length calculations often yield order $b=O(n^{1/3})$ under specific objectives and dependence assumptions. We use the scheme as a black-box, input perturbation ensemble for a frozen TSFM rather than reproving its properties.

\section*{Appendix B: Algorithms and Implementation Details}

Algorithms~\ref{alg:smr2_residual_concise} and~\ref{alg:smr2_infer_concise} use the same stride view construction as Eqs.~\eqref{eq:seg_mean}--\eqref{eq:upsample}: for the current window length $m$, the construction pools only the first $s\lfloor m/s\rfloor$ positions into complete stride $s$ blocks over all $d{+}1$ channels, repeats each block statistic $\lfloor m/\lfloor m/s\rfloor\rfloor$ times, and forward fills or truncates to return a length $m$ view.

\textbf{Algorithm 1 SMR\textsuperscript{2} Residual Extraction (Concise)}
\begin{algorithm}[h]
\caption{SMR\textsuperscript{2} Residual Extraction Phase}
\label{alg:smr2_residual_concise}
\begin{algorithmic}[1]
\Require Training set $\{(\mathbf{c}_t,y_t)\}$, stride set $\mathcal{S}$, context length $L_{\text{train}}$
\Ensure Residual histories $\{\mathcal{R}^{(s)}\}_{s\in\mathcal{S}}$ and per stride ridge correctors $f^{\text{res}}$
\For{$s\in\mathcal{S}$}
  \For{$t\in\mathcal{T}_{\text{train}}$ with $t>L_{\text{train}}$}
    \State Construct $\mathbf{X}_{t-1,L_{\text{train}}}^{(s)}$ from $\mathbf{X}_{t-1,L_{\text{train}}}$ using Eqs.~\eqref{eq:seg_mean}--\eqref{eq:upsample} with stride $s$
    \State $\hat{y}_t^{(s)}\!\leftarrow\!f_{\theta}^{\text{base}}(\mathbf{X}_{t-1,L_{\text{train}}}^{(s)})$
    \State $\epsilon_t^{(s)}\!\leftarrow\!y_t-\hat{y}_t^{(s)}$
    \State Append $\epsilon_t^{(s)}$ to $\mathcal{R}^{(s)}$
  \EndFor
  \State Fit ridge corrector $f^{\text{res}}$ on $\mathcal{R}^{(s)}$ (once, then frozen)
\EndFor
\end{algorithmic}
\end{algorithm}

\textbf{Algorithm 2 SMR\textsuperscript{2} Inference}
\begin{algorithm}[h]
\caption{SMR\textsuperscript{2} Inference Phase}
\label{alg:smr2_infer_concise}
\begin{algorithmic}[1]
\Require Context $\mathbf{X}_{t} = (\mathbf{c}_{t-L+1:t}, y_{t-L+1:t})$, per-stride correctors $\{f^{\text{res}}\}$, recent residuals $\{\mathcal{R}^{(s)}\}$, validation metrics $\{R^2(s), \mathrm{MSE}(s)\}_{s\in\mathcal{S}}$
\Ensure Ensemble forecast $\tilde{y}_{t+1}^{\mathrm{ens}}$
\For{$s\in\mathcal{S}$}
  \State Construct $\mathbf{X}_t^{(s)}$ from $\mathbf{X}_t$ using Eqs.~\eqref{eq:seg_mean}--\eqref{eq:upsample} with stride $s$
  \State $\hat{y}_{t+1}^{(s)} \leftarrow f_{\theta}^{\text{base}}(\mathbf{X}_t^{(s)})$
  \State $\boldsymbol{\phi}_t^{(s)} \leftarrow [\,\epsilon_{t-q+1:t}^{(s)},\,\psi(\mathbf{X}_t^{(s)})\,]$
  \State $\hat{\epsilon}_{t+1}^{(s)} \leftarrow f^{\text{res}}(\boldsymbol{\phi}_t^{(s)})$
  \State $\tilde{y}_{t+1}^{(s)} \leftarrow \hat{y}_{t+1}^{(s)} + \hat{\epsilon}_{t+1}^{(s)}$
\EndFor
\State Compute nonnegative simplex weights $\{w^{(s)}\}_{s\in\mathcal{S}}$ according to Eq.~\eqref{eq:r2weights}
\State $\tilde{y}_{t+1}^{\mathrm{ens}} \leftarrow \sum_{s\in\mathcal{S}} w^{(s)} \tilde{y}_{t+1}^{(s)}$
\end{algorithmic}
\end{algorithm}

\textbf{Algorithm 3 MBB Wrapper}
\begin{algorithm}[h]
\caption{MBB Moving Block Bootstrap}
\label{alg:mbb_concise}
\begin{algorithmic}[1]
\Require Residuals $\{r_j\}_{j=1}^{n}$ from $\mathcal{T}_{\text{val}}$, block length $b$, samples $K$
\Ensure Ensemble point forecast $\hat{y}_{t+1}^{\mathrm{MBB}}$
\For{$i=1$ to $n-b+1$}
  \State Form block $B_i=\{r_i,\ldots,r_{i+b-1}\}$
\EndFor
\For{$k=1$ to $K$}
  \State Draw blocks with replacement to form sequence $r^{(k)}_{1:L}$
  \State Perturb history and predict $\hat{y}_{t+1}^{(k)}$
\EndFor
\State $\hat{y}_{t+1}^{\mathrm{MBB}}\!\leftarrow\!\frac{1}{K}\!\sum_k\!\hat{y}_{t+1}^{(k)}$
\end{algorithmic}
\end{algorithm}

\section*{Appendix C: Hyperparameter Configuration}

We summarize below the key hyperparameters used for SMR\textsuperscript{2} and MBB. SMR\textsuperscript{2} uses strides $\mathcal{S}=\{1,2,4,6,8\}$ aligned with dominant spectral components of monthly SPEI.
Residual correction applies ridge regression ($\lambda_{\text{ridge}}=0.1$, residual lag length $q=6$) on a feature vector of recent residuals and stride $s$ input view summary statistics, with normalization $\varepsilon=10^{-6}$, and ensemble weights based on validation $R^2(s)$.  
MBB adopts block length $b=\lfloor n^{1/3} \rfloor$ and $K=100$ bootstrap samples, averaged into the ensemble point forecast.
All models are frozen during inference, and hyperparameters are tuned only on validation data.

\section*{Appendix D: Dataset and Evaluation Protocol}

\textbf{Data sources.}
All datasets used in this study are publicly available. Meteorological predictors are obtained from the NCEP--NCAR Reanalysis~1 dataset. Drought targets use the SPEI-30 variable from the global daily SPEI-GD database (0.25$^\circ$, native span 1982--2021), converted to monthly cadence. The study region covers South Australia, with representative grid points at
$(-26.125^\circ,129.125^\circ)$,
$(-29.125^\circ,134.875^\circ)$, and
$(-35.625^\circ,138.875^\circ)$.
The SPEI product is accessible at
\url{https://doi.org/10.5281/zenodo.8060268}.
Documentation and downloads for NCEP--NCAR Reanalysis~1 are available from NOAA at
\url{https://www.psl.noaa.gov/data/gridded/data.ncep.reanalysis.html}.
Public climate indices (e.g., ENSO, IOD) are provided by operational centers such as NOAA CPC.

\textbf{Temporal coverage.}
Monthly data from 1982--2018, using the SPEI-30 variable from the daily SPEI-GD product, which natively spans 1982--2021.

\textbf{Preprocessing.}
Missing values are linearly interpolated and predictors standardized. The target is the SPEI-30 series from SPEI-GD, converted to the monthly cadence used for evaluation; no additional target smoothing is applied beyond this cadence conversion.

\textbf{Evaluation protocol.}
All models are evaluated with rolling one-step-ahead forecasting over the held out test period. For all Deep Learning methods implementation, the \href{https://github.com/thuml/Time-Series-Library}{THUML Time-Series-Library} \citep{thuml_tslib} provides consistent data loaders, training scripts, and reference configurations that cover most baselines. For our method implementation, we provide it in our code repo.

\section*{Appendix E: Computational Cost and Deployment Notes}

\textbf{Runtime and hardware.}
Experiments are run on a single workstation equipped with an NVIDIA GeForce RTX~3090 GPU with 24\,GB of memory, an Intel~12700K CPU, 64\,GB of DDR4 3200 memory, and a 2\,TB NVMe solid state drive. The operating system is Ubuntu~20.04.3 LTS.
Average inference latency is below 3\,s per site, with memory usage under 4\,GB.

\textbf{Energy and sustainability.}
Because SMR\textsuperscript{2} and MBB operate exclusively at inference time, they avoid the repeated forward and backward passes over many epochs that fine-tuning requires, and instead add only a small number of additional forward passes (the $K$ bootstrap evaluations for MBB and the $|\mathcal{S}|$ stride evaluations for SMR\textsuperscript{2}) to a single frozen model inference. Estimating the cost of fine-tuning as $2\times$ the forward cost per step (the backward pass) over the typical number of update steps used for the deep learning baselines shows that the wrappers avoid the dominant training time computation; the exact energy saving depends on the fine-tuning budget and hardware. We therefore report measured latency and memory, and treat energy reduction as a qualitative deployment benefit rather than a measured percentage.

\textbf{Deployment.}
Both wrappers are implemented as lightweight, plug-and-play modules for the frozen foundation model backbones evaluated in this paper (TabPFN, TimeGPT, and TimesFM), and can run efficiently on CPU only systems.

\section*{Appendix F: Reproducibility Checklist}

\textbf{Code and data availability.}
All scripts, processed data, and documentation are available at:

\url{https://github.com/Wentao-Gao/Lightweight_Wrappers}

\textbf{Software environment.}
The software stack uses Python~3.10, PyTorch~2.0.1 with CUDA~11.8, NumPy~1.24.3, Pandas~2.0.2, and Scikit learn~1.2.2. Random seeds are fixed for NumPy and PyTorch so that runs are deterministic where supported. The seed affects trainable baselines and MBB resampling; frozen foundation model backbones are deterministic at inference.

\textbf{Reproducibility scripts.}
The code repository includes configuration files and single command scripts that reproduce all scores and figures reported in the paper without retuning.

\end{document}